\documentclass[12pt]{article}
\usepackage[letterpaper]{geometry}
\usepackage{amsmath,amssymb,amsthm}
\usepackage{booktabs}
\usepackage{xtab}
\usepackage{algorithm}
\usepackage{url}
\usepackage{amsmath}
\theoremstyle{plain}
\newtheorem{thm}{Theorem}
\newtheorem{lem}[thm]{Lemma}
\newtheorem{prop}[thm]{Proposition}

\newtheorem{remark}[thm]{Remark}

\newcommand{\MLP}[1]{(MLP'$_{#1}$)}

\newcommand{\PSLP}{P_e}
\newcommand{\PLP}{P_{e^*}}

\title{On optimizing over lift-and-project closures}
\author{Pierre Bonami\thanks{Supported by ANR grant ANR06-BLAN-0375}\\
{\em \small Laboratoire d'Informatique Fondamentale, CNRS/Aix Marseille Universit\'e, Marseille, France}
}

\begin{document}
\maketitle
\begin{abstract}
The lift-and-project closure is the relaxation obtained by computing all lift-and-project cuts from the initial formulation of a mixed integer linear program or equivalently by computing all mixed integer Gomory cuts read from all tableau's corresponding to feasible and infeasible bases. In this paper, we present an algorithm for approximating the value of the lift-and-project closure. The originality of our method is that it is based on a very simple cut generation linear programming problem which is obtained from the original linear relaxation by simply modifying the bounds on the variables and constraints. This separation LP can also be seen as the dual of the cut generation LP used in disjunctive programming procedures with a particular normalization. We study some properties of this separation LP in particular relating it to the equivalence between lift-and-project cuts and Gomory cuts shown by Balas and Perregaard. Finally, we present some computational experiments and comparisons with recent related works.

\end{abstract}

\section{Introduction}
In this paper, we consider Mixed Integer Linear Programs (MILP) of the form:
\begin{equation*}
\label{MILP}
\begin{aligned}
&\max\, {c'}^T x\\
& s.t.\\
& A'x \ge b, \\
& x\in \mathbb Z^{p}_+ \times \mathbb R^{n-p}_+,
\end{aligned}
\tag{MILP}
\end{equation*}
where $A'$ is an $m \times n$ rational matrix, $c' \in \mathbb Q^n$ and $b \in \mathbb Q^m$.

In the last 20 years, cutting plane methods have become one of the main ingredients for solving MILP. State of the art solvers are branch-and-cut algorithms that employ a number of different cutting plane algorithms to automatically strengthen formulations: Gomory Mixed Integer (GMI) Cuts \cite{gomory:63}, Mixed Integer Rounding cuts \cite{Nemhauser.Wolsey:90}, Knapsack Covers \cite{Balas.Zemel:78}, Flow-covers \cite{Padberg.Van-Roy.ea:85,Roy.Wolsey:87}. The positive impact of these different methods on the solution algorithms for MILP have been evaluated in many independent computational studies for example \cite{bixby.et.al:04}.

A central concept in cutting planes for MILP is that of elementary closures. Simply put, the elementary closure of a given family of cuts is the relaxation obtained by applying all cuts of the said family that can be obtained using the initial formulation of MILP only (without a recursive application). Two well known closures are the Chv\'atal-Gomory closure \cite{Chvatal:73} obtained by applying all possible Gomory Fractional cuts \cite{Gomory:60} and the split closure  \cite{cook.kannan.schrijver:90}. Elementary closure have been a fruitful concept from a theoretical standpoint. In particular, four important results are the facts that the Chv\'atal-Gomory closure is a polyhedron \cite{Chvatal:73} on which it is NP-hard to optimize \cite{eisenbrand:99} and that the split closure is also a polyhedron \cite{cook.kannan.schrijver:90} on which it is NP-Hard to optimize \cite{caprara.letchford:03}. More recently, closures have also been considered from a computational point of view. Several authors have computed empirically the strength of different closures on problem instances from the literature:
Fischetti and Lodi have computed the Chv\'atal-Gomory closure \cite{Fischetti.Lodi:07}, Bonami and Minoux the lift-and-project closure \cite{Bonami.Minoux:05}, 
Bonami et. al. the projected Chv\'atal-Gomory closure \cite{Bonami.Cornuejols.Dash.et.al:08}, Balas and Saxena the split closure \cite{Balas.Saxena:08} and Dash, G\"unluk and Lodi the MIR closure \cite{Dash.Gunluk.Lodi:10}. These studies have shown that elementary closure often give rise to strong relaxation of MILPs. In particular, the Split and MIR closures, which are equivalent, often close a very significant portion of the integrity gap, even though computing time are prohibitive. In that view an interesting line of research is to devise faster algorithms that could give a fairly good approximation of the split closures.

In this paper, we are interested in approximating the split closure by looking at simpler (and weaker) closures that are defined by two families of cuts: GMI cuts read from LP tableau's and {\em strengthened lift-and-project cuts} \cite{Balas.Ceria.et.al:93}. It was shown by Balas and Perregaard \cite{Balas.Perregaard:03} that both classes of cuts are equivalent and define the same closure which we call {\em lift-and-project closure}. The lift-and-project closure is a polyhedron but, to the best of our knowledge, it is not known whether it can be optimized in polynomial time. A weaker closure that can be optimized in polynomial time is the elementary lift-and-project closure on which one can separate in polynomial time by using disjunctive programming.  Our goal in this paper is to devise efficient algorithms to compute the elementary lift-and-project closure and to approximate heuristically the lift-and-project closure.

Our approach is similar to the one proposed in \cite{Bonami.Minoux:05}: we optimize the elementary lift-and-project closure by a cutting plane algorithm and strengthen each cut individually to approximate the lift-and-project closure. Our main contribution with respect to \cite{Bonami.Minoux:05} is that we use a different cut generation oracle. While in \cite{Bonami.Minoux:05} a linear program in an extended space is solved to generate each cut, here we solve a linear program in the original space of variables which we call {\em membership LP}. Furthermore, we are able to establish some nice properties of this membership LP, and show that it is typically solved in very few iterations.

Our work is related to \cite{Balas.Perregaard:03}. First the membership LP is related to disjunctive programming and the classical Cut Generation LP (CGLP). It can be seen as the dual of a simplified version of the CGLP used in \cite{Balas.Perregaard:03}. The simplicity of its structure allows us to give a new view on the equivalence between strengthened lift-and-project cuts and GMI cuts that was shown in \cite{Balas.Perregaard:03}. It also gives further insights into this equivalence. In particular a separation algorithm was proposed in \cite{Balas.Perregaard:03} to separate lift-and-project cuts in the LP tableau of the linear relaxation. One limitation of the algorithm is that it only works by improving an existing cut, our separation procedure removes this limitation. 

Our work is also related to several recent computational studies. Dash and Goycoolea \cite{Dash.Goycoolea:10} and Fischetti and Salvagnin \cite{Fischetti.Salvagnin:10} proposed two different methods for separating rank-1 GMI cuts from alternate bases of the LP relaxation. Our work can be seen as a third proposal for finding rank-1 GMI cuts. We also note that the same separation LP that we use, was independently used in a recent work by Fischetti and Salvagnin \cite{Fischetti.Salvagnin:10b} to also optimize on the simple lift-and-project closure.

In Section 2, we give formal definitions of the various cutting planes and closures we are interested in. In section 3, we introduce the membership LP and establish its relations to disjunctive programming. In Section 4, we establish various properties of the solutions of the membership LP and relate it to GMI cuts. In Section 5, we discuss the practical solution of the membership LP and its use for optimizing the lift-and-project closure. Finally, in Section 6, we present computational results.

\section{Definitions and basic results}
The Linear Programming (LP) relaxation of (MILP) obtained by dropping all integrity constraints is
\begin{equation*}
\label{LP}
\begin{aligned}
&\max\, {c'}^T x\\
& s.t.\\
& A'x \ge b, \\
& x\in \mathbb R^n_+.
\end{aligned}
\tag{LP}
\end{equation*}
We also consider the standard form of \eqref{LP} where slack variables are introduced:
\begin{equation*}
\label{LP:std}
\begin{aligned}
&\max\, {c}^T x\\
& s.t.\\
& Ax = b, \\
& x\in \mathbb R^{n+m}_+.
\end{aligned}
\tag{SLP}
\end{equation*}
In \eqref{LP:std}, we assume that the slack variables are the first $m$ ones ($x_1,\ldots,x_m$) and the structural variables are the last $n$ ($x_{m+1},\ldots, x_{m+n}$).
With this ordering of the variables $A=(-I\,A')$ (where $I$ is the identity matrix) and $c^T = (0^T,\,{c'}^T)$. 
We denote by $N:=\{m+1,\ldots,m+n\}$ the index set of structural variables in \eqref{LP:std}, by $N^I := \{m+1, \ldots, m+p\}$ the index set of integer constrained structural variables and by $M :=\{1,\ldots,m\}$
the set of slack variables (or the set of constraints). 
Given a set $J \subseteq M \cup N$, we denote by $x_J$ the vector with component in $J$ and by
$A^J$ the sub-matrix of $A$ composed by the columns $i \in J$ ($A'=A^N$). Given $J \subseteq M$, we denote by $A_J$ the sub-matrix
of $A$ composed by the rows of $A$ indexed by $J$. By abuse of notations we will denote $A^{\{j\}}$ and $A_{\{i\}}$ by $A^j$ and $A_i$ respectively.

We denote the index set of variables in a basis of LP by $B:=\{i_1,\ldots,i_m\} \subseteq \{1,\ldots,n+m\}$ and the index set of nonbasic variables by $J:=\{1,\ldots,n+m\}\setminus B$. Given the index set of a basis $B$;
$A^B$ is the basis matrix (we recall that variables indexed by $B$ form a basis whenever $A^B$ is non-singular), the LP tableau corresponding to $B$ is given by
\begin{equation*}
x_B = {A^B}^{-1} b - {A^B}^{-1} A^J x_J.
\end{equation*}
Finally, we denote $\overline A = {A^B}^{-1} A$ and $\overline b = {A^B}^{-1}b$.

Next, we define the four cutting planes used in this paper: simple intersection cut, GMI cut, lift-and-project cut and strengthened lift-and-project cut. 

\subsection{Cutting planes}
\label{sec:cuts}
Before proceeding to the definitions of the various cutting planes, we briefly remind the concepts of disjunctive and split cuts that are central to the four methods. 

In MILP, we call disjunction a condition of the form $A^1 x \ge b^1 \vee A^2 x \ge b^2 \vee \ldots \vee A^q x \ge b^q$ that is satisfied by all solutions to \eqref{MILP} (i.e. each solution of \eqref{MILP} verifies at least one of the systems of inequalities). For example, if $(\pi,\pi_0) \in \mathbb Z^{n+1}$ is such that $\pi_{\{p+1,\ldots,n\}} = 0$ the {\em split disjunction} $\pi^T x \le \pi_0 \vee \pi^T x \ge \pi_0 + 1$ is valid for \eqref{MILP}. All inequalities presented here, result from split disjunctive in that, for a given $(\pi,\pi_0)$ they are valid inequalities for the set:
\begin{equation*}
P^{(\pi,\pi_0)} := \text{conv} \left(
\left\{x \in \mathbb R^n_+ : A'x \ge b, \pi^T x \le \pi_0 \right\} \cup 
\left\{x \in \mathbb R^n_+ : A'x \ge b, \pi^T x \ge \pi_0 + 1 \right\}
\right).
\end{equation*}
As such, all the inequalities presented belong to the family of {\em split cuts}\cite{cook.kannan.schrijver:90}.
\paragraph{The simple intersection cut}
Let $B$ be a basis of \eqref{LP:std} and $i_k \in B \cap N^I$ be such that $\overline b_k \not \in \mathbb Z$. The row of the simplex tableau corresponding to $i_k$ is
\begin{equation}
\label{row}
x_{i_k} + \sum\limits_{j \in J} \overline a_{kj} x_j = \overline b_{k}
\end{equation}
Let $f_{0} = \overline b_{k} - \lfloor \overline b_k \rfloor$.
The {\em intersection cut} \cite{Balas:79} from the convex set $\{x \in R^{n+m} : \lfloor \overline b_k \rfloor \le x_{i_k} \le \lceil \overline b_k \rceil\}$ applied to \eqref{row} is
\begin{equation}
\label{eq:inter}
\sum\limits_{j \in J} 
\max\left\{
\overline a_{kj}(1- f_{0}),\,
-\overline a_{kj}f_{0} \right\} x_j \ge f_{0}(1- f_{0})
\end{equation}
This cut is valid for all solution to \eqref{MILP} and is also known as the {\em simple disjunctive cut} from the condition $x_{i_k} \le \lfloor \overline b_k \rfloor \vee x_{i_k} \ge \lceil \overline b_k \rceil$.

\paragraph{Gomory Mixed Integer Cut}
The simple intersection cut can be strengthened using the integrity of the variables in $J \cap N^I$. 
For each $j \in J \cap N^I$ we define 
$f_j := \overline a_{kj} - \lfloor \overline a_{kj} \rfloor$.
We define $\pi \in \mathbb Z^{m+n}$ as
\begin{equation*}
\pi_j := \begin{cases}
\lfloor \overline a_{kj} \rfloor & \text{ if } j \in J \cap N^I \text{ and } f_j \le f_0,\\
\lceil \overline a_{kj} \rceil & \text{ if } j \in J \cap N^I \text{ and } f_j > f_0\\
1 & \text{ if } j = i_k,\\
0 & \text{otherwise.}
\end{cases}
\end{equation*}
The Gomory Mixed-Integer Cut is the intersection cut from the convex set $\{x \in R^{n+m} : \lfloor \overline b_k \rfloor \le \pi^T x \le \lceil \overline b_k \rceil\}$ applied to \eqref{row}:
\begin{equation}
\label{GMI}
\sum\limits_{\substack{j \in J \cap N^I\\ f_j \le f_0}}
f_j(1- f_{0}) x_j +
\sum\limits_{\substack{j \in J \cap N^I\\ f_j > f_0}}
(1-f_{j})f_{0} x_j +
\sum\limits_{j \in J \setminus N^I} 
\max\left\{
\overline a_{kj}(1- f_{0}),\,
-\overline a_{kj}f_{0} \right\} x_j \ge f_{0}(1- f_{0})
\end{equation}

We note that the GMI cut dominates all intersection cuts obtained from row \eqref{row} and convex sets of the form $\{x \in R^{n+m} : \lfloor \overline b_k \rfloor \le \pi^T x \le \lceil \overline b_k \rceil\}$ with $\pi_{(J \setminus N^I) \cup (B \setminus \{i_k\}) } = 0$ and $\pi_{i_k} = 1$ (see \cite{Andersen.Cornuejols.ea:05} for a proof).

\paragraph{The lift-and-project cut}
The lift-and-project cuts are the valid inequalities for the sets $P^{(e_k,\pi_0)}$ taken for any $k = 1,\ldots,p$ and $\pi_0 \in \mathbb Z$, 
obtained by intersecting the polyhedron $\{x \in \mathbb R^n_+ : A'x \ge b\}$ with a simple disjunction of the form $x_k \le \pi_0 \vee x_k \ge \pi_0 + 1$.

The fundamental theorem of Balas \cite{Balas:98,Balas:85} on unions of polyhedra allows to formulate the separation problem for $P^{(e_k,\pi_0)}$ as a linear program. 

\begin{thm}[\cite{Balas:98,Balas.Ceria.et.al:93}]
\label{thm:balas}
Let $\hat x \in \mathbb R^n_+$. $\hat x \in P^{(\pi,\pi_0)}$ if and only if the solution to
the Cut Generation LP
\begin{displaymath}
\label{CGLP}
\begin{aligned}
&\min \alpha^T \hat x - \beta \\
&s.t.:\\
&\alpha = u^T A^N + s - u_0 \pi,\\
&\alpha = v^T A^N + t + v_0 \pi,\\
&\beta = u^T b - u_0 \pi_0,\\
&\beta = v^T b + v_0 (\pi_0 + 1),\\
&\alpha \in \mathbb R^n, \beta \in \mathbb R,\, u,v \in \mathbb R^{m}_+,\, s,\, t \in \mathbb R^n_+,\, u_0, v_0 \in \mathbb R_+,\\
\end{aligned}
\end{displaymath}
is non-negative.
\end{thm}
If the solution the cut generation linear program is negative, $\alpha^T x \ge \beta$ defines  a valid inequality in the space of structural variables that cuts off the point $\hat x$.

Lift-and-project cuts are separated by solving the cut generation linear program for simple disjunctive (i.e. with $\pi = e_k$). Since this program is unbounded whenever it has a negative solution, a {\em normalization constraint} is often added to cut the cone of feasible solution.
Following \cite{Balas.Perregaard:02}, the most commonly used constraint is $\sum_{i \in M} (u_i + v_i) + \sum_{i \in N} (s_i + t_i) + u_0 + v_0 = 1$. In the remainder, we denote by (CGLP) the cut generation linear program of Theorem \ref{thm:balas} augmented with that normalization condition.
\paragraph{The strengthened lift-and-project cut}
Lift-and-project cuts can be strengthened in a similar fashion as simple intersection cuts can be strengthened to GMI cuts.

Consider a basic solution $(\hat \alpha, \hat \beta, \hat u, \hat v, \hat s, \hat t, \hat u_0, \hat v_0)$ to (CGLP), the coefficients $\hat \alpha$ verify:
\begin{equation*}
\hat \alpha_k = \max\{\hat u A^k - \hat u_0, \hat v A^k + \hat v_0\}
\end{equation*}
and
\begin{equation*}
\hat \alpha_j = \max\{\hat u A^j, \hat v A^j \},\,\,\,j\neq k 
\end{equation*}

These coefficients can be strengthened by replacing $\hat \alpha_j$, for $j \in \{1,\ldots,p\} \setminus \{k\}$, with
\begin{equation*}
\overline \alpha_j := \min\{\hat u A^j - \hat u_0 \lfloor m_j \rfloor, \hat v A^j + \hat v_0 \lceil m_j \rceil \}
\end{equation*}
where
\begin{equation*}
m_j := \frac{\hat uA^j - \hat vA^j}{\hat u_0 + \hat v_0},
\end{equation*}
(see \cite{Balas.Jeroslow:80,Balas.Ceria.et.al:96*1} for justifications.)
\subsection{Equivalences and Closures}
\label{sec:closures}
Balas and Perregaard showed that simple intersection cuts are equivalent to lift-and-project cuts, and that GMI cuts are equivalent to strengthened lift-and-project cuts \cite{Balas.Perregaard:03}.
The equivalence can be stated as follows: every simple intersection cut (resp. GMI cut) from a basic solution (feasible or not) of \eqref{LP:std} can be derived as a lift-and-project cut (resp. strengthened lift-and-project cut) obtained from a basic feasible solution of (CGLP); and conversely every lift-and-project cut from a basic feasible solution of (CGLP) is an intersection cut.
Here similar means that after the two cuts are put into the same space by eliminating slack variables they are identical up to multiplication by a positive scalar. The equivalence gives the precise correspondence between the two bases of \eqref{LP:std} and (CGLP). In view of this result, solving (CGLP) can be interpreted as finding a good basis from which to generate a GMI cut.

We note that this result is stated in \cite{Balas.Perregaard:03} for 0-1 MILP only but its generalization to MILP is straightforward.

Generally speaking, the elementary closure of a family of cuts consists of all cuts in the family that can be obtained directly from the initial formulation.
A direct consequence of the result mentioned above is that lift-and-project cuts and simple intersection cuts on one hand and GMI cuts and strengthened lift-and-project cuts on the other hand are equivalent in terms of their closures \cite{Cornuejols.Li:01}.

The simple lift-and-project closure is obtained by taking all cuts defined by basic feasible solutions to (CGLP). We denote it by $\PSLP$. Geometrically, It is obtained by intersecting all polyhedra $P^{(e_k,\pi_0)}$ for all $k\in\{1,\ldots,p\}$ and all $\pi_0 \in \mathbb Z$:
\begin{equation*}
\PSLP := \bigcap\limits_{\substack{k\in\{1,\ldots,p\}\\ \pi_0 \in \mathbb Z}} P^{(e_k,\pi_0)}
\end{equation*}
Equivalently, $\PSLP$ is obtained by taking all simple intersection cuts from all bases (feasible and infeasible) of \eqref{LP:std}. 
We denote by $\mathcal B$ the set of all bases of $\eqref{LP:std}$. For each $B \in \mathcal B$, we define $P(B)$ the cone obtained by dropping the non-negativity requirements of the variables $j \not\in B$:
$$
P(B) := \{x \in \mathbb R^{n+m} : x_{B} = \overline b - \overline A^J x_J,\,\, x_J \ge 0 \}.
$$
The simple lift-and-project closure can be defined alternatively as
$$
\PSLP = \text{proj}_N \left( \bigcap\limits_{B \in \mathcal B} \bigcap\limits_{k \in N^I \cap B} \text{conv}\left(P\left(B\right) \cap \left\{x \in \mathbb R^{n+m} : x_{m+k} \le \lfloor \overline b_k \rfloor \vee x_{m+k} \le \lceil \overline b_k \rceil \right\}\right)\right).
$$
(where proj$_N$ denotes the projection onto the space of structural variables).

The lift-and-project closure is defined by taking all strengthened lift-and-project cuts obtained from all basic feasible solutions to (CGLP). We denote it by $\PLP$. It is more difficult to give a precise and intuitive geometrical definition of it than for $\PSLP$. By using the equivalence between lift-and-project cuts and GMI cuts, $\PLP$ is also obtained by taking all GMI cuts defined from all bases of \eqref{LP:std}. More formally, for each basis $B \in \mathcal B$ and $k \in N^I \cap B$, let 
$\Pi^B_k := \{ (\pi,\pi_0) \in \mathbb Z^{n+1}: \pi_k = 1, \pi_{N\setminus N^I \cup B\setminus \{k\}} = 0,\, \pi_0 = \lfloor \overline b_k \rfloor \}.$ $\Pi^B_k$ be the subset of splits that are considered for strengthening the simple intersection cut from basis $B$ and variable $k$ (i.e. all valid splits which have zero coefficients for all basic variables except one). The lift-and-project closure is the intersection of all strengthened intersections cut obtained from these larger subset of splits:
$$
\PLP := \text{proj}_ N \left( \bigcap\limits_{B \in \mathcal B} \bigcap\limits_{k \in N^I \cap B}  \bigcap\limits_{(\pi,\pi_0) \in \Pi^B_k} \text{conv}\left(P\left(B\right) \cap \left\{x \in \mathbb R^{n+m} : \pi^T x \le \pi_0 \vee \pi^T x \ge \pi_0 + 1 \right\}\right)\right).
$$

Since both $\PSLP$ and $\PLP$ are defined by a finite number of inequalities both are clearly polyhedra. Moreover, $\PSLP$ can be optimized in polynomial time. To the best of our knowledge, it is not known if $\PLP$ can be optimized in polynomial time.

Finally, let us remind that the split closure $P_S$ is obtained by intersecting all split cuts:
$$
P_S :=  \bigcap\limits_{\substack{(\pi,\pi_0)\in \mathbb Z^{n} \times \mathbb Z \\ \pi_{\{p+1,\ldots,n\}}=0}} P^{(\pi,\pi_0)}.
$$
As shown by \cite{cook.kannan.schrijver:90}, $P_S$ is a polyhedron. Andersen, Cornu\'ejols and Li \cite{andersen:05} have shown that $P_S$ can also be defined in terms of intersection cuts:
$$
P_S = \text{proj}_N \left( \bigcap\limits_{B \in \mathcal B} \bigcap\limits_{\substack{(\pi,\pi_0)\in \mathbb Z^{n} \times \mathbb Z \\ \pi_{\{p+1,\ldots,n\}}=0}} \text{conv}\left(P\left(B\right) \cap \left\{x \in \mathbb R^{n+m} : \pi^T x \le \pi_0 \vee \pi^T x \ge \pi_0 + 1 \right\}\right)\right).
$$
 From the definitions, it should be clear that $P_S \subseteq \PLP \subseteq \PLP$, since $\PLP$ and $\PSLP$ only use a subset of splits and elementary disjunctions belong to $\Pi^B_k$. It is NP-Hard to optimize over $P_S$. Computational studies such as \cite{Balas.Saxena:08,Dash.Gunluk.Lodi:10} have shown that it can be a strong relaxation, but computational times are often prohibitive (in the two studies computing $P_S$ often takes more time than solving \ref{MILP} to optimality). This justifies our motivation to optimize $\PLP$ and $\PSLP$.

\section{The membership LP}
\label{sec:membership}
In this section, we establish a linear program which given a split disjunction $(\pi,\pi_0)$ and a point $\hat x \in \mathbb R^n$ answers the question of whether $\hat x \in P^{(\pi,\pi_0)}$. It should appear immediately that this question can already be answered by solving (CGLP) using Theorem \ref{thm:balas}. Our linear program is indeed similar to (CGLP), by duality, but it does not require the introduction of new variables and is formulated in the same space as \eqref{LP}. We will mostly use the results of this section in the context of simple disjunctive (i.e. when $\pi = e_k$), but since they are not restricted to simple disjunctive, we present them in the more general context of split disjunctive.

Our membership LP is based on the following proposition.
\begin{prop}
\label{thm:compact_l-and-p}
Let $P=\{x \in \mathbb R^n_+ : A'x \ge b\}$, $\hat x \in P$ and $(\pi_0,\pi) \in \mathbb Z^{n+1}$ be such that
$\pi_0 < \pi^T \hat x < \pi_0 +1$.
$\hat x \in \text{conv} \left(P \cap \left( \left\{\pi^T x \le \pi_0 \right\} \cup \left\{\pi^T x \ge \pi_0 + 1 \right\} \right) \right)$
if and only if there exists $y \in \mathbb R^{n}$ such that
\begin{equation}
\label{appa}
\begin{aligned}
&\pi^T y - (\pi_0 + 1) \left( \pi \hat x - \pi_0 \right) \ge 0\\
&0 \le A' y - b \left( \pi \hat x - \pi_0 \right) \le A' \hat x - b\\
&0 \le y \le \hat x \\
\end{aligned}
\end{equation}
\end{prop}
This proposition can be proved in different ways. In particular, it is closely related to classical results of disjunctive programming. We will discuss in the following paragraphs the connections between $\eqref{appa}$ and (CGLP). First, we give a short and self contained proof of it since it provides a simple way to understand \eqref{appa}.
\begin{proof}
\begin{itemize}
\item [{\em (i)}] Suppose that $\hat x \in P^{(\pi_0, \pi)}$. Since $\pi^T \hat x - \pi_0 \in ] 0,1[$, there exists $x^0 \in P \cap \left\{\pi^T x \le \pi_0 \right\}$,
$x^1 \in P \cap \left\{\pi^T x \ge \pi_0 + 1\right\}$ and $\lambda \in ] 0,1 [$ such that $\hat x = \lambda x^1 + (1-\lambda) x^0$. 
Without loss of generality, we can assume that $\pi^T x^0 = \pi_0$, $\pi^T x^1 = \pi_0 + 1$. Then, since $\pi^T (\lambda x^1 + (1-\lambda) x^0) = \pi^T \hat x$, we have
$$
\lambda = \pi^T \hat x - \pi_0.
$$
We take $\overline y = \lambda x^1$ and verify that it satisfies \eqref{appa}.

First, $\pi^T \overline y - (\pi_0 + 1) (\pi^T \hat x - \pi_0) = \lambda(\pi_0 + 1) - (\pi_0 + 1) \lambda = 0$.
Also, $A \overline y - b (\pi^T \hat x - \pi_0) = \left( A x^1 - b \right) \lambda \ge 0$ and $\overline y = \lambda x^1 \ge 0$, since $Ax^1 \ge b$ , $x^1 \ge 0$ and $\lambda \ge 0$. Furthermore, since $\hat x = \overline y + (1 - \lambda) x^0$ we have:
$$
A \overline y - \lambda b + (1-\lambda) (A x^0 - b) = A \hat x - b.
$$
and therefore, since $A x^0 - b \ge 0$ and $\lambda \le 1$
$$
A \overline y - \lambda b \le A \hat x - b.
$$
$y \le \hat x$ in the same manner.

\item[{\em (ii)}] Suppose now that $\overline y$ satisfies \eqref{appa}. Let $\lambda = \pi \hat x - \pi_0$. Note that $0 < \lambda < 1$ by hypothesis. We take $ x^1 = \frac{\overline y}{\lambda}$
and $ x^0 = \frac{\hat x -\overline y}{1 - \lambda}$. Clearly $\hat x = \lambda  x^1 + (1 - \lambda)  x^0$.
Furthermore $ x^1 \in P$ and $\pi^T x^1 \ge \pi_0 + 1$.\\

It remains to show that $x^0 \in P \cap \{\pi^T x \le \pi_0\}$. $x^0 \in P$, since $(1 - \lambda)(A x^0 - b) = A (\hat x -\overline y) - b + b \lambda = A \hat x - b - A \overline y + b \lambda \ge 0$ and $(1 - \lambda) x^0 = \hat x -\overline y \ge 0$. Finally
$$\pi^T x^0 - \pi_0 = \frac{\pi^T \hat x -\pi^T \overline y - (1 - \lambda) \pi_0}{1 - \lambda} = \frac{-\pi^T \overline y + \lambda + \lambda \pi_0}{1 -\lambda}
= -\frac{\pi^T \overline y - (\pi_0 + 1)\lambda}{1 - \lambda} \le 0 
$$
\end{itemize}
\end{proof}
By applying Proposition \ref{thm:compact_l-and-p}, we can formulate a linear program which answers the question of membership for $\hat x$ and $P^{(\pi_0,\pi)}$:
\begin{equation*}
\label{MLP}
\begin{aligned}
\max \,\,&\pi^T y - (\pi_0 + 1) \left( \pi \hat x - \pi_0 \right)\\
&\text{s.t.:}\\
&0 \le A'y - b \left( \pi \hat x - \pi_0 \right) \le A'\hat x - b,\\
&0 \le y \le \hat x.
\end{aligned}
\tag{MLP}
\end{equation*}
As a direct application of the proposition, $\hat x \in P^{(\pi_0,\pi)}$ if and only if the optimal solution to \eqref{MLP} is non-negative.

An alternative way to prove the theorem is to see that \eqref{MLP} is closely related to the dual of (CGLP) with the normalization condition $u_0 + v_0 = 1$:
\begin{equation*}
\label{CGLPb}
\begin{aligned}
&\min \alpha^T \hat x - \beta \\
&\text{s.t.:}\\
&\alpha = u^T A' + s - u_0 \pi,\\
&\alpha = v^T A' + t + v_0 \pi,\\
&\beta = u^T b - \pi_0 u_0,\\
&\beta = v^T b + (\pi_0 + 1) v_0,\\
&u_0 + v_0 = 1,\\
&u,v \in \mathbb R^{m}_+,\, s,\, t \in \mathbb R^n_+,\, u_0, v_0 \in \mathbb R_+\\
\end{aligned}
\tag{CGLP'}
\end{equation*}
A basic results stated in \cite{Balas.Perregaard:03} (Lemma 1) is that in any solution of (CGLP) that yields an inequality $\alpha^T x \ge \beta$ that is not dominated by the constraint of \eqref{LP}, both $u_0$ and $v_0$ are positive. The following lemma shows that in fact, if we assume $\hat x \in P$ and $\pi_0 < \pi^T \hat x < \pi_0 +1$, the conditions $u_0 , v_0 \ge 0$ can be dropped from \eqref{CGLPb}.

\begin{lem}
\label{lem:bp}
Let $\hat x \in P$ be such that $\pi_0 < \pi^T \hat x < \pi_0 +1$. Suppose that $(\alpha, \beta, u,v,u_0,v_0)$ verifies the following system of equations:
\begin{equation*}
\begin{aligned}
&\alpha = u^T A' + s - u_0 \pi,\\
&\alpha = v^T A' + t + v_0 \pi,\\
&\beta = u^T b - \pi_0 u_0,\\
&\beta = v^T b + (\pi_0 + 1) v_0,\\
&u_0 + v_0 = 1,\\
&u,v \in \mathbb R^{m}_+,\, s,\, t \in \mathbb R^n_+,\, u_0, v_0 \in \mathbb R\\
\end{aligned}
\end{equation*}
with $u_0 \le 0$ or $v_0 \le 0$. Then $\alpha^T \hat x \ge \beta$.
\end{lem}
\begin{proof}
The proof if similar for $u_0 \le 0$ and $v_0 \le 0$; we only consider the case $u_0 \le 0$.

Let $(\alpha, \beta, u,v,u_0,v_0)$ be a vector that satisfies the conditions of the statement. Then, by definition of $(\alpha, \beta, u,v,u_0,v_0)$
$$\alpha^T \hat x - \beta = (u^T A' + s - u_0 \pi) \hat x - u^T b + \pi_0 u_0 = u^T (A' \hat x - b) - u_0 (\pi^T \hat x - \pi_0).$$
Now the last quantity is non-negative since $u \ge 0$, $A' \hat x - b$, $\pi^T \hat x > \pi_0$ and $u_0 \le 0$.
\end{proof}
Since we are only interested in negative solutions to \eqref{CGLPb}, by application of the lemma, we can therefore remove the conditions $u_0,v_0 \ge 0$. 

After removing non-negativity constraints on $u_0$ and $v_0$, we can eliminate $\alpha$, $\beta$ and $v_0$ from \eqref{CGLPb}:
\begin{equation*}
\begin{aligned}
&\min (u^T A') \hat x + (s - u_0 \pi)^T \hat x - u^T b - \pi_0 u_0,\\
&\text{s.t.:}\\
&(u-v)^T A' + (s - t) = \pi\\
&(v-u)^T b  - u_0 = - \pi_0 -1 \\
&u,v \in \mathbb R^{m}_+,\, s,\, t \in \mathbb R^n_+,\, u_0 \in \mathbb R\\
\end{aligned}
\end{equation*}
Finally we eliminate $u_0$:
\begin{equation}
\label{CGLP:simple}
\begin{aligned}
&\min u^T (A \hat x -b) + s^T \hat x + (\pi^T \hat x - \pi_0)
\left( (u-v)^T b  - \pi_0 - 1 \right),\\
&\text{s.t.:}\\
&(u-v)^T A + (s - t) = \pi\\
&u,v \in \mathbb R^{m}_+,\, s,\, t \in \mathbb R^n_+.\\
\end{aligned}
\end{equation}
Taking the dual, with multipliers $y$ associated to the constraints, one obtains \eqref{MLP}. It should be clear from this development that a dual feasible solution to \eqref{MLP} with negative dual cost gives rise to a valid cut for $P^{(\pi, \pi_0)}$.

\begin{remark}
\label{remark:non_valid}
Note that one should be careful when dropping the non-negativity conditions on $u_0$ and $v_0$ in \eqref{CGLPb}. This has indeed one important consequence: while non-negative solutions to \eqref{CGLPb} define valid (non-violated) inequalities, non-negative solution to \eqref{CGLP:simple} do not. For example, the trivial solution $u=v=0$, $s-t=\pi$ is always feasible for \eqref{CGLP:simple} but the inequality it defines is not necessarily valid. For example take $\pi = e_k$ and $\pi_0 = 1$, the solution $u=v=0$, $t=0$, $s=e_k$ defines the inequality $x_k \le 2$ which certainly can not be valid for all integer programs! Nevertheless if the solution is negative, it gives a valid cut.
\end{remark}


\eqref{CGLPb} (with that normalization condition) has been studied previously in \cite{Balas.Ceria.et.al:93,Balas.Perregaard:02,Fischetti.Lodi.Tramontani:10}.
A well known property of the normalization condition $u_0 + v_0 =1$ is that if $\hat x$ is an extreme point of $P$ such that $0 < \hat x_k - \pi_0 < 1$ then an optimal solution to \eqref{CGLPb} with $\pi = e_k$ is the intersection cut obtained from a basis defining $\hat x$ and the corresponding strengthened cut is the GMI cut from the same basis (see \cite{Fischetti.Lodi.Tramontani:10} for a precise statement and a complete proof). In the next proposition, we re-state the property in the context of \eqref{MLP}. Looking at \eqref{MLP}, we are able to give a slightly more precise view: if $\hat x$ is an extreme point, \eqref{MLP} has a unique solution. We present a proof of the results in the context of \eqref{MLP} since it is short and gives insights into the connections between \eqref{LP} and \eqref{MLP} (the fact that the associated cut is the intersection cut will follow directly from the results of the next section, we delay the proof until then).

\begin{prop}
\label{lemma:extr}
Let $P=\{x \in \mathbb R^n_+ : Ax \ge b\}$, and $\hat x$ be an extreme point of $P$ such that
$0 < \pi^T \hat x - \pi_0 < 1$, then the unique solution to 
\begin{equation}
\label{mp:m}
\begin{aligned}
\max \,\,&\pi^T y - (\pi_0 + 1) (\pi \hat x - \pi_0)\\
&\text{s.t.:}\\
&0 \le A'y - b(\pi \hat x - \pi_0) \le A'\hat x - b\\
&0 \le y \le \hat x \\
\end{aligned}
\end{equation}
is $\overline y = \hat x(\pi^T \hat x - \pi_0)$.
Furthermore $\pi^T \overline y - (\pi_0 + 1) (\pi \hat x - \pi_0) < 0$.
\end{prop}

\begin{proof}
First, we prove that $\overline y$ is a solution. Since $0 < \pi \hat x - \pi_0 < 1$, and $A \hat x - b \ge 0$:
$$
0 \le (\pi \hat x - \pi_0)(A \hat x - b) \le A \hat x -b
$$
and also $0 \le(\pi \hat x - \pi_0) \hat x \le \hat x$.\\

Second, we show that this solution is unique. We denote by $A^= x = b^= , \, x^= = 0$, the subset of the inequalities defining $P$ satisfied at equality by $\hat x$. Since $\hat x$ is an extreme point of $P$, it is the unique solution of $A^= x = b^= , \, x^= = 0$. A solution of \eqref{mp:m} satisfies:
$$
\begin{aligned}
A^= y - b^= (\pi^T x - \pi_0) = 0\\
y^= = 0.
\end{aligned}
$$
This system in turn has a unique solution.

Finally, it is trivial to check that $\pi^T  \overline y - (\pi_0 + 1) (\pi \hat x - \pi_0) < 0$:
$$
\pi^T  \overline y - (\pi_0 + 1) (\pi \hat x - \pi_0) = (\pi \hat x - \pi_0 - 1) (\pi \hat x - \pi_0)
$$
which is the product of a positive and a negative number and therefore always negative.
\end{proof}

This Lemma has important consequences for the practical usefulness of \eqref{MLP}. In particular, it exposes its main weakness with respect to classical lift-and-project cut generation LPs with more general normalization condition such as the one used in \cite{Balas.Perregaard:03}.

The standard way to apply lift-and-project cuts (and GMI cuts), is to generate them recursively by rounds where at each round the current basic optimal solution to the LP relaxation (computed by the simplex algorithm) is cut by separating one lift-and-project cut for each basic fractional integer constrained variable. At the end of the round all the lift-and-project cuts generated are added to the LP relaxation and the process is iterated recursively. In such an algorithm, using \eqref{MLP} to generate cuts brings little novelty since the point to cut is always an extreme point of the polyhedron used to separate and the cuts generated would therefore be GMI cuts (and they could certainly be read more quickly from the optimal LP tableau).

Nevertheless, there are contexts where solving \eqref{MLP} could present a practical interest. In particular one can not always assume that $\hat x$ is an extreme point of $P$. This is the case, for example, if the objective function in \eqref{MILP} is replaced with a nonlinear convex function, or if the LP relaxation is solved with an interior point method. In this work, we will focus on another case of interest: generating cuts of lowest rank. If in the procedure outlined above, one does not use the previously generated cuts to generate new cuts, then the point to cut is not an extreme point of the polyhedron used to separate.

\section{Correspondences and equivalences}
Our goal in this section is to relate the cuts that can be obtained from the solution of \eqref{MLP} with intersection and GMI cuts.
We consider only elementary disjunctions of the form $x_k \le \pi_0 \vee x_k \ge \pi_0 + 1$. Note that this assumption is not restrictive since any split $\pi^T x \le \pi_0 \vee \pi^T x \ge \pi_0 + 1$ can always be put into this form by defining a new variable to be equal to $\pi^T x$.
For this elementary disjunction, the membership LP is
\begin{equation*}
\label{MP:sd}
\begin{aligned}
\max \,\,&y_k - \lceil \hat x_k \rceil (\hat x_k - \lfloor \hat x_k \rfloor ),\\
&\text{s.t.:}\\
&0 \le A'y - b (\hat x_k - \lfloor \hat x_k \rfloor ) \le A'\hat x - b,\\
&0 \le y \le \hat x,\\
&y \in \mathbb R^n.
\end{aligned}
\tag{MLP'$_k$}
\end{equation*}
For ease of notation, in the sequel, we denote $(\hat x_k - \lfloor \hat x_k \rfloor )$ by $f_k$.

We are interested in dual feasible solutions of \eqref{MP:sd} since they define cuts for our \eqref{MILP}.
The dual of \eqref{MP:sd} is the Cut Generation Linear Program
\begin{equation*}
\label{MP:dual}
\begin{aligned}
&\min u^T (A' \hat x -b) + s^T \hat x + (u-v)^T b  f_k - \lceil \hat x_k \rceil f_k\\
&\text{s.t.:}\\
&(u-v)^T A + (s - t) = e_k\\
&u,v \in \mathbb R^{m}_+,\, s,\, t \in \mathbb R^n_+.\\
\end{aligned}
\tag{CGLP'$_k$}
\end{equation*}

A solution $(u,v,s,t)$ to \eqref{MP:dual} defines a cut of the form:
\begin{equation}
\label{cut:lp}
u^T A' x + s^T x + \left( \left(u-v\right)^T b - \lceil \hat x_k \rceil  \right) x_k \ge u^T b + \left( \left(u-v\right)^T b  - \lceil \hat x_k \rceil \right) \lfloor  \hat x_k \rfloor 
\end{equation}
in the space of structural variables.

Finally, \eqref{MP:sd} in standard form is
\begin{equation*}
\label{MP:std}
\begin{aligned}
\max \,\,&y_{m+k} - \lceil \hat x_k \rceil f_k\\
&\text{s.t.:}\\
&A y = A'\hat x - b + b f_k,\\
&0 \le y_N \le \hat x_N,\\
&0 \le y_M \le A'\hat x_N - b,\\
&y \in \mathbb R^n+m.
\end{aligned}
\tag{SMLP$_k$}
\end{equation*}
\eqref{MP:std} is in the form of a general LP with lower and upper bounds on the variables. In a basic solution, we denote by
$J^+$ the set of variables which are nonbasic at their upper bound (corresponding to variables $u$ and $s$ in the dual) and $J^-$ the set of variables which are nonbasic at their lower bound (corresponding to variables $v$ and $t$). We recall that a basic solution of \eqref{MP:std} is dual feasible if the reduced costs of every variables in $J^-$ is non-positive and the reduced cost of all variables in $J^+$ is non-negative.

An immediate property of \eqref{MP:std}, is that all of its bases are bases of \eqref{LP:std}.

\begin{prop}
The variables indexed by $B$ form a base of \eqref{MP:std} if and only if they form a base of \eqref{LP:std}.
\end{prop}
\begin{proof}
Since both problems have the same matrix of constraints this is true by definition of bases.
\end{proof}

We now turn to the main result of this section which is to characterize the dual feasible bases of \eqref{MP:std} and to relate them to intersection cuts for \eqref{LP:std}. First, in the next lemma, we characterize the dual feasible bases of \eqref{MP:std}.

\begin{lem}
\label{thm:equiv}
Let $B=\{i_1,\ldots,i_m\}$ be the index set of a basis of \eqref{MP:std}. $B$ is dual feasible if either: 
\begin{itemize}
\item[{\em(i)}] the variable $y_{m+k}$ is non-basic at its upper bound, or
\item[{\em(ii)}] $y_{m+k}$ is basic with corresponding tableau row $y_{m+k} + \sum_{j \in J} \overline a_{ij} y_j = \overline a_{i0}$ in \eqref{MP:std} and $j \in J$ is non-basic at its upper bound if $\overline a_{ij} < 0$ and nonbasic at its lower bound if $\overline a_{ij} > 0$.
\end{itemize}
\end{lem}
\begin{proof}
Let $B$ be a basis, $J^+$ be the set of nonbasic variables at their upper bound and $J^-$ the set of nonbasic variables at their lower bound. 

First, we prove the case {\em (i)}. 
If $m+k \not\in B$ then $c_B=0$, therefore all variables have a zero reduced cost except $x_{m+k}$ which has a reduced cost of $1$ and if $m+k$ is nonbasic, it has to be at its upper bound for the solution to be dual feasible.

Now, we suppose that $m+k \in B$. We assume that $x_{m+k}$ is basic in row 
$i$. The reduced cost of a variable $j \in J$ is given by $\overline c_j = c_j - c_B {A^{B}}^{-1} A^j = 0 - \overline a_{ij}$. Therefore, $B$ is dual feasible if for all $j \in J$ such that $\overline a_{ij} > 0$, $j \in J^-$ and for all $j \in J$ such that $\overline a_{ij} < 0$, $j \in J^+$. 
\end{proof}

Now we study the cuts associated to dual feasible solutions of \eqref{MLP}. Given a dual feasible basis $B$ of \eqref{MLP}, we will call the cut \eqref{cut:lp} associated to the solution of \eqref{MP:dual} defined by $B$, the cut associated to $B$. Since we use $B$ both as a basis for \eqref{LP:std} and \eqref{MP:std}, we need to take particular care in denoting the right-hand-sides of the two associated LP tableau's.
We will denote by $\overline b := {A^B}^{-1} b$, the right hand side in the tableau of \eqref{LP:std} and by $\overline a_0 := {A^B}^{-1} (A'\hat x - b f_k)$ the right hand side in the tableau of \eqref{MP:std}.

Using the reduced cost computed in Theorem \ref{thm:equiv}, we can compute the values for the dual variables. Using these values the cut \eqref{cut:lp} 
can be computed. Keep in mind that by Lemma \ref{lem:bp} and remark \ref{remark:non_valid}, only dual solutions of negative cost give rise to valid inequalities.

Note that in case {\em (i)}, when $y_{m+k}$ is non-basic, all dual variables have value $0$ except $s_k$ which is equal to $1$. In that case, the dual objective value of the solution is alway non-negative:
\begin{equation*}
\hat x_k - \lceil \hat x_k \rceil ( \hat x_k  - \lfloor \hat x_k\rfloor ) = \lfloor \hat x_k\rfloor (\lceil \hat x_k \rceil - \hat x_k) \ge 0,
\end{equation*}
and therefore no valid inequality can be generated (note that this case is the one exhibited in remark \ref{remark:non_valid}).

The only case of interest is therefore when $y_{m+k}$ is basic. In that case, the values of the dual variables in the basic solution are given by the reduced costs computed above: 
\begin{align}
\label{eq:dual_sol_1}
&u_j =
\begin{cases}
- \overline a_{ij} &\text{if } j \in J^+ \cap M,\\
0 & \text{if } j \in M \setminus J^+
\end{cases} 
\\
\label{eq:dual_sol_2}
&v_j = 
\begin{cases}
\overline a_{ij} &\text{if } j \in J^- \cap M\\
0 & \text{if } j \in M \setminus J^-
\end{cases}\\
\label{eq:dual_sol_3}
&s_j =
\begin{cases}
- \overline a_{ij} &\text{if } j \in J^+ \cap N,\\
0 & \text{if } j \in N \setminus J^+
\end{cases} 
\\
\label{eq:dual_sol_4}
&t_j = 
\begin{cases}
\overline a_{ij} &\text{if } j \in J^- \cap N\\
0 & \text{if } j \in N \setminus J^-
\end{cases}
\end{align}

In the following, we will denote by $\overline x \in \mathbb R^{n+m}$ the primal basic solution of \eqref{LP:std} corresponding to basis $B$
First, we show that the dual solution can have a negative cost only if $\overline x_{m+k} \in \left] \left\lfloor \hat x_k \right\rfloor, \left\lceil \hat x_k \right\rceil \right[$.

\begin{lem}
\label{lem:x_k}
Let $B$ be the index set of a dual feasible basis of \eqref{MP:std} such that ${m+k} \in B$. Let $\overline x \in \mathbb R^{n+m}$ be the primal basic solution of \eqref{LP:std} corresponding to basis $B$.
Suppose that $\overline x_{m+k} \not\in  \left] \lfloor \hat x_k \rfloor, \lceil \hat x_k \rceil \right[$, then $\hat x$ is not cut by the dual solution to \eqref{MP:std}.
\end{lem}
\begin{proof}
By Lemma \ref{lem:bp} the dual solution gives a cut for $\hat x$ only if $u_0$ and $v_0$ are positive.
$u_0$ is given by $u_0 = (u - v)^T b + \lceil \hat x_{m+k} \rceil$. Note that $(u-v)^T b = -\sum_{i \in M} \overline a_{ij} b = -\overline x_{m+k}$. Therefore $u_0 > 0$ only if $\overline x_{m+k} < \lceil \hat x_{m+k} \rceil$. Similarly since $v_0 = 1 - u_0$, $v_0 > 0$ only if $\overline x_{m+k} > \lfloor \hat x_{k} \rfloor$.
\end{proof}

Lemma \ref{lem:x_k} indicates that the only relevant dual feasible bases for valid cutting planes are those where $\overline x_k \in  \left] \lfloor \hat x_k \rfloor, \lceil \hat x_k \rceil \right[$. Next theorem shows that for those dual feasible bases, the cut derived is the intersection cut obtained from the row of the tableau of \eqref{LP:std} corresponding to the same basis and the same basic variable.
The row of the tableau of \eqref{LP:std} corresponding to $x_{m+k}$ reads
$$
x_{m+k} + \sum\limits_{j \in J} \overline a_{ij} x_j = \overline b_i
$$

\begin{thm}
\label{thm:eqi_inter}
Let $B$ be the index set of a dual feasible basis of \eqref{MP:std} such that ${m+k} \in B$. Let $\overline x \in \mathbb R^{n+m}$ be the primal basic solution of \eqref{LP:std} corresponding to this basis.
If $\overline x_{m+k} \in  \left] \lfloor \hat x_k \rfloor, \lceil \hat x_k \rceil \right[$, the cut obtained from the dual solution of \eqref{MP:std} corresponding
to the basis indexed by $B$ is equivalent to the simple intersection cut for $x_k$ of \eqref{LP:std} obtained from the basis indexed by $B$.
\end{thm}

\begin{proof}
We consider the cut \eqref{cut:lp} defined by the dual solution. To show that \eqref{cut:lp} is equivalent to the intersection cut, we first need to write it in the space with slack variables.
\begin{equation*}
u^T A' x_N + s^T x_N + \left( \left(u-v\right)^T b - \lceil \hat x_{k} \rceil  \right) x_{m + k} \ge u^T b + \left( \left(u-v\right)^T b  - \lceil \hat x_k \rceil \right) \lfloor  x_k \rfloor.
\end{equation*} We start by replacing $A'x_N -b$ with $x_M$:
$$
u^T x_M + s^T x_N + \left( \left(u-v\right)^T b - \lceil \hat x_{k} \rceil  \right) x_{m + k} \ge \left( \left(u-v\right)^T b  - \lceil \hat x_k \rceil \right) \lfloor  x_k \rfloor.
$$
Using the values of $u$ and $s$ given by \eqref{eq:dual_sol_1} and \eqref{eq:dual_sol_3} this can be rewritten as:
$$
- \sum\limits_{j\in J^+} \overline a_{ij} x_{j} + \left( \left(u-v\right)^T b - \lceil \hat x_{k} \rceil  \right) x_{m + k} \ge \left( \left(u-v\right)^T b  - \lceil \hat x_k \rceil \right) \lfloor  x_k \rfloor. 
$$
Noting that $(u-v)^T b = ({A^B}^{-1} b)_i = \overline b_i$ the cut can be rewritten as:
$$
- \sum\limits_{j\in J^+} \overline a_{ij} x_{j} +  \left( \overline b_i - \lceil \hat x_{k} \rceil  \right) x_{m + k} \ge \left(\overline b_i  - \lceil \hat x_k \rceil \right) \lfloor  x_k \rfloor. 
$$
Now using the tableau row of \eqref{LP:std} $x_{m+k} + \sum_{j \in J}\overline a_{ij} x_j = \overline b_i$, we eliminate $x_{m+k}$:
$$
- \sum\limits_{j\in J^+} \overline a_{ij} x_{j} + \left( \overline b_i - \lceil \hat x_{k} \rceil  \right)  (\overline b_i - \sum\limits_{j\in J} \overline a_{ij} x_{j}) \ge \left(\overline b_i  - \lceil \hat x_k \rceil \right) \lfloor  x_k \rfloor.
$$
Finally, re-grouping the coefficient, we obtain:
\begin{equation}
\label{almost}
- (\overline b_i - \lfloor \hat x_k \rfloor) \sum\limits_{j\in J^+} \overline a_{ij} x_{j} + (\lceil \hat x_k \rceil - \overline b_i) \sum\limits_{j\in J^-} \overline a_{ij} x_{j} \ge (\overline b_i - \lfloor \hat x_k \rfloor) (\lceil \hat x_k \rceil - \overline b_i) .
\end{equation}

Since $\overline b_i = \overline x_{m+k} \in \left] \lfloor \hat x_k \rfloor, \lceil \hat x_k \rceil \right[$, we have $\overline b_i - \lfloor \hat x_{k} \rfloor = \overline b_i - \lfloor \overline b_i \rfloor$, defining $f_0 := \overline b_i - \lfloor \overline b_i \rfloor = \overline x_{m+k} - \lfloor \overline x_{m+k} \rfloor$ as in \eqref{eq:inter}, we have that \eqref{almost} is equivalent to 
\begin{equation}
\label{there}
- f_0 \sum\limits_{j\in J^+} \overline a_{ij} x_{j} + (1 - f_0) \sum\limits_{j\in J^-} \overline a_{ij} x_{j}\ge (1 - f_0)  \overline f_0.
\end{equation}
This last cut is identical to the intersection cut \eqref{eq:inter}. Indeed,  if $j \in J^+$, $\overline a_{ij} \le 0$ implies that $\max\{-f_0 \overline a_{ij}, (1 - f_0) \overline a_{ij}\} = -f_0 \overline a_{ij}$; and if $j \in J^-$, $\overline a_{ij} \ge 0$ implies that $\max\{-f_0 \overline a_{ij}, (1 - f_0) \overline a_{ij}\} = (1-f_0) \overline a_{ij}$. This ends the proof of the Theorem.
\end{proof}

Theorem \ref{thm:eqi_inter} implies that every cut obtained from a negative basic feasible solution \eqref{MP:dual} is the simple intersection cut obtained from a tableau row of \eqref{LP:std}. We now consider the strengthened cuts obtained from \eqref{MP:dual} and show that they are GMI cuts.

The strengthened cuts obtained from \eqref{MP:dual} are obtained by applying the strengthening operation defined in Section \ref{sec:cuts}. To apply the operation, we need to transform our solution of \eqref{MP:dual} defined by (\ref{eq:dual_sol_1})--(\ref{eq:dual_sol_4}) to the corresponding solution of (CGLP). This is simply done by taking $u_0 = (u -v)^T A + \pi_0 + 1$, $v_0 = 1 - u_0$, $\alpha = u^T A + s -u_0 e_k$ and $\beta = u^T b$. One can easily check that if $u,v$ is feasible for \eqref{MP:dual}, this solution is feasible for (CGLP).

\begin{thm}
\label{thm:eqi_GMI}
Let $B$ be the index set of a dual feasible basis of \eqref{MP:std} such that ${m+k} \in B$. Let $\overline x \in \mathbb R^{n+m}$ be the primal basic solution of \eqref{LP:std} corresponding to basis $B$.
If $\overline x_{m+k} \in  \left] \lfloor \hat x_k \rfloor, \lceil \hat x_k \rceil \right[$, the strengthened lift-and-project cut obtained from the dual solution of \eqref{MP:std} corresponding
to $B$ is equivalent to the GMI cut obtained from basis $B$.
\end{thm}

\begin{proof}
To establish the equivalence, we need to write the strengthened lift-and-project cut obtained from the basis $B$. We recall, that the strengthening consist in replacing the coefficient $\hat \alpha_j := \max \{ u^T A^j, v^T A^j\}$ with $\overline \alpha_j$ for $j=1,\ldots, p$, $j\neq k$ where $\overline \alpha_j$ is given by
\begin{equation*}
\overline \alpha_j := \min\{u A^j - u_0 \lfloor m_j \rfloor, v A^j + v_0 \lceil m_j \rceil \}
\end{equation*}
where $u, v, u_0$ and $v_0$ are feasible for (CGLP) and
\begin{equation*}
m_j := \frac{ u A^j - v A^j}{ u_0 + v_0}
\end{equation*}
Therefore to obtain the strengthened cut, we need to add $\overline \alpha_j - \hat \alpha_j$ to the coefficient of $x_j$ obtained in \eqref{there} for $j=1,\ldots, p$, $j\neq k$. Next, we verify that doing so one obtains the GMI cut \eqref{GMI}.

Our first order of business is to build the solution to (CGLP), we take $ u$ and $v$ given by the dual solution corresponding to basis $B$ as computed in \eqref{eq:dual_sol_1}--\eqref{eq:dual_sol_4}. $ u_0$ is then given by $u_0 = (u -v)^T A + \pi_0 + 1 = \pi_0 + 1 - \hat x_{m+k} = 1 - f_0$ (where, as before $f_0 = \hat x_{m+k} - \lfloor x_{m+k} \rfloor$) and $v_0 = 1 - u_0 = f_0$.

Using the definition of \eqref{MP:dual}, we can rewrite $m_j$ as $m_j = ( t_j -  s_j)$. From this, it is immediate that for $j \in B$, since $s_j = t_j =0$, $m_j = 0$. Furthermore, from the constraints of \eqref{MP:dual} we then have $(u - v)^T A^j = 0$ and therefore $\overline \alpha_j = \hat \alpha_j$. This implies that the coefficient of the cut is not changed for $j \in B$.

We now consider $j \in J \cap N^{I}$. In this case, using the value of the dual solution given by \eqref{eq:dual_sol_1}--\eqref{eq:dual_sol_4}, we have $m_j = ( t_j - s_j) = \overline a_{ij}$
Summing up everything, so far, we have that the formula for $\hat \alpha_j$ for $j \in \{1,\ldots,p\} \setminus B$ is given by
\begin{equation*}
\overline \alpha_j = \min\{ uA^j - (1 - f_0) \lfloor \overline a_{ij} \rfloor, vA^j + f_0 \lceil \overline a_{ij} \rceil \}.
\end{equation*}

To establish the strengthened coefficients, we consider separately the case $j \in J^+$ and the case $j \in J^-$.

First, assume $j \in J^+ \cap N^{I}$. Then $\overline a_{ij} \le 0$, $s_j = - \overline a_{ij}$ and $t_j = 0$. From $(v - u)^T A^j = (s_j - t_j) = - \overline a_{ij} \ge 0$, we have 
$\hat \alpha_j =  v A^j$ and $uA^j = v A^j + \overline a_{ij}$ therefore:
\begin{multline*}
\overline \alpha_j - \hat \alpha_j = \min \left\{vA^j + \overline a_{ij} - \left(1 - f_0 \right) \left\lfloor \overline a_{ij} \right\rfloor, 
vA^j + f_0 \left\lceil \overline a_{ij} \right\rceil \right\} - vA^j = \\
\min \left\{\overline a_{ij} - \left(1 - f_0 \right) \left\lfloor \overline a_{ij} \right\rfloor,  f_0 \left\lceil \overline a_{ij} \right\rceil \right\}.
\end{multline*}
Since $j \in J^+$, the coefficient for $x_j$ in \eqref{there} is $-f_0 \overline a_{ij}$, the new coefficient is therefore:
\begin{multline*}
-f_0 \overline a_{ij} + \min \left\{\overline a_{ij} - \left(1 - f_0\right) \left\lfloor \overline a_{ij} \right\rfloor,  
f_0 \left\lceil \overline a_{ij} \right\rceil \right\}\\
= \min \left\{-f_0 \overline a_{ij} + \overline a_{ij} - \left(1 - f_0 \right) \left\lfloor \overline a_{ij} \right\rfloor, 
-f_0 \overline a_{ij} + f_0 \left\lceil \overline a_{ij} \right\rceil \right\}\\
= \min \left\{ \left(1-f_0\right) \left( \overline a_{ij} - \left\lfloor \overline a_{ij} \right\rfloor \right), f_0 \left( \left\lceil \overline a_{ij} \right\rceil - \overline a_{ij} \right) \right\}
\end{multline*}

Now, we consider the case $j \in J^- \cap N^{I}$. Then $\overline a_{ij} \ge 0$, $s_j = 0$ and $t_j = \overline a_{ij}$. From $(v - u)^T A^j = (s_j - t_j) = - \overline a_{ij} \le 0$, we have 
$\hat \alpha_j = u A^j$ and $vA^j = u A^j - \overline a_{ij}$ therefore:
\begin{multline*}
\overline \alpha_j - \hat \alpha_j = \min \left\{uA^j - \left(1-f_0 \right) \left\lfloor \overline a_{ij} \right\rfloor, uA^j -  \overline a_{ij} + f_0 \left\lceil \overline a_{ij} \right\rceil \right\} - uA^j = \\
\min \left\{ - \left(1-f_0 \right) \left\lfloor \overline a_{ij} \right\rfloor, - \overline a_{ij} + f_0 \left\lceil \overline a_{ij} \right\rceil \right\}.
\end{multline*}
Since $j \in J^-$, the coefficient for $x_j$ in \eqref{there} is $(1-f_0) \overline a_{ij}$, the new coefficient is therefore:
\begin{multline*}
\left( 1-f_0 \right) \overline a_{ij} + \min \left\{ - \left(1-f_0 \right) \left\lfloor \overline a_{ij} \right\rfloor, - \overline a_{ij} +  f_0 \left\lceil \overline a_{ij} \right\rceil \right\}\\
= \min \left\{ \left(1-f_0 \right) \overline a_{ij} - \left( 1-f_0 \right) \left\lfloor \overline a_{ij} \right\rfloor,  \left(1-f_0 \right) \overline a_{ij} - \overline a_{ij} +f_0 \left\lceil \overline a_{ij} \right\rceil \right\}\\
= \min \left\{ \left(1 - f_0\right) \left( \overline a_{ij} - \left\lfloor \overline a_{ij} \right\rfloor \right), f_0 \left( \left\lceil \overline a_{ij} \right\rceil - \overline a_{ij} \right) \right\}.
\end{multline*}

Therefore, for both $j \in J^+ \cap N^{I}$ and $j \in J^- \cap N^{I}$ we have that the coefficient for the strengthened cut is 
\begin{equation}
\min \left\{ (1 - f_0) \left( \overline a_{ij} - \lfloor \overline a_{ij} \rfloor \right), f_0 \left( \lceil \overline a_{ij} \rceil - \overline a_{ij} \right) \right\}.
\end{equation}
It is trivial to check that this coefficient is the same as the GMI cut coefficient.
\end{proof}

Note that this theorem allows to complete the results of Lemma \ref{lemma:extr}, by adding the fact that the cut obtained is indeed the intersection cut or GMI cut.

Theorems \ref{thm:eqi_inter} and \ref{thm:eqi_GMI} can be seen as the counterparts of Theorems 4A, 4B and 5 in \cite{Balas.Perregaard:03}. Indeed, the statements are almost equivalent except that here the simpler normalization condition $u_0 + v_0 = 1$ is used here whereas the normalization condition $\sum (u_i + v_i) + \sum (s_i+t_i) + u_0 + v_0 = 1$ was used in \cite{Balas.Perregaard:03}. In \cite{Balas.Bonami:09}, the theorems of \cite{Balas.Perregaard:03} were generalized to weighted normalizations of the form $\sum \psi_i (u_i + v_i) + \sum \xi_i (s_i+t_i) + u_0 + v_0 = \psi_0 $ where $\psi \in \mathbb R^m_+$ , $\xi \in R^n_+$ and $\psi_0 \in \mathbb Z_+$. Theorems \ref{thm:eqi_inter} and \ref{thm:eqi_GMI} can be seen as a particular case of those latter theorems (but the proofs here are simpler and more direct). Another important result in \cite{Balas.Perregaard:03} was an algorithm to solve (CGLP) in the tableau of \eqref{LP:std}. Our results show that, with the simpler normalization $u_0 + v_0 = 1$, a similar algorithm to solve \eqref{CGLPb} is simply to solve \eqref{MP:sd}. We note that a limitation of the algorithm in \cite{Balas.Perregaard:03} is that it starts from a basis of \eqref{LP:std} that gives a cut. If such a basis is not available, the solution of \eqref{MP:sd} can be seen as an easy way to find one.
\section{Optimizing over $\PSLP$ and approximating $\PLP$}
We now turn back to our initial goal which is the optimization of $\PSLP$ and $\PLP$. To do so, we follow a simple Kelley cutting plane algorithm \cite{Kelley:60} where the master problem is the LP relaxation of \eqref{MILP} augmented with cuts and the cut generation procedure is the solution of \eqref{MP:sd}. More precisely, at each iteration, we first solve the master problem and obtain a solution $x^*$. If $x^*$ is integer feasible, we stop. Otherwise we solve the separation problem \eqref{MP:sd} for a number of fractional components of $x^*$ which should be integer. If some of the separation problems lead to a cut (i.e. has negative objective value), we add those cuts to the master problem and iterate. Otherwise, if all \eqref{MP:sd} for all $k \in \{1,\ldots,p\}$ such that $x^*_k \not\in \mathbb Z$ have non-negative objective value, we have a proof that $x^* \in \PSLP$. To approximate $\PLP$, we simply strengthen every cut derived by the usual strengthening procedure as done in the proof of Theorem \ref{thm:eqi_GMI}. Note that cuts are never added to \eqref{MP:sd} in the course of the procedure (i.e. we stick to generating rank one cuts only). This simple algorithm should terminate in a finite number of iterations (since there are only a finite number of rank 1 cuts) but there is no complexity guarantee. Also concerning $\PLP$ there is no guarantee of how good is the approximation. Nevertheless, it is guaranteed that approximating $\PLP$ in this way will give a better bound than optimizing $\PSLP$ and since the extra computational cost is negligible, it should always be better.

Note that this approach is similar to the one presented in \cite{Bonami.Minoux:05} although it is presented there from a different point of view. In \cite{Bonami.Minoux:05}, a Benders decomposition is applied to a linear programming model of $\PSLP$ defined in a higher dimensional space. Applying this decomposition amounts to performing the cut generation method outlined above. The main difference between the two approaches is the separation problem that is solved. In \cite{Bonami.Minoux:05}, a (CGLP) is solved while here we solve \eqref{MLP} (other differences are the details of our algorithm that are described below). 

In order to make this algorithm somewhat efficient, one drawback has to be taken care of. In the course of the algorithm, it often happens that although $x^*_k$ is fractional the solution of \eqref{MP:sd} does not lead to a cut. 
If at a given iteration, this happens for all $k \in \{1,\ldots,p\}$, this gives the useful information that $x^* \in \PSLP$ (and stops the algorithm) but if it only happens for some $k$'s it only indicates that $x^* \in P^{(e_k,\lfloor x^*_k \rfloor)}$ but does not help the algorithm in progressing. After the first iteration, this phenomenon usually happens very frequently, and if not taken care of, a significant part of the computation can be spent in solving those separation problems that do not lead to a cut. To try to overcome it, at each iteration we solve \eqref{MP:sd} only for those $k$ which led to a cut in the previous iteration. Of course, to maintain the validity of the algorithm, in the case when no cuts were generated, we need to test all the variables (we also test all the variables in the case where some tailing off is detected). This strategy is obviously not perfect as still some problems \eqref{MP:sd} will not lead to cuts but it greatly improves the efficiency of the procedure (more comments on the practical performance will be made in Section \ref{sec:comp_plp}).
The pseudo-code of the algorithm we apply to approximate $\PLP$ is given in Algorithm \ref{alg:cut}.

\begin{algorithm}
\begin{tabular}{cp{5.3in}}
0. & {\bf Initialize feasible region.}\\
   & $\mathcal C \leftarrow \mathbb \{x \in R^n_+ : A'x \ge b\}$.\\
1. & {\bf Initialize variables list}\\
   & $\mathcal K \leftarrow \{1,\ldots,p\}$, {\tt reinit} $\leftarrow true$.\\
2. & {\bf Master Problem}\\
   & Solve $\max\{{c'}^T x : x \in  C\}$. Let $\hat x$ be the solution.\\
3. & {\bf Filter variables}\\
   &$\mathcal F \leftarrow \{i \in \mathcal K: \hat x_i - \lfloor \hat x_i \rfloor \ge \epsilon \}$, $\mathcal K \leftarrow \emptyset$. Sort $\mathcal F$ by increasing values of $\hat x_k$.\\
4. & {\bf Separate}\\
   & For each $k \in \mathcal F$ solve \eqref{MP:sd}. If the solution is $\le -\epsilon$ construct strengthened cut $\alpha^T x \ge \beta$ from its dual solution and let $\mathcal K \leftarrow K \cup \{k\}$, $\mathcal C \leftarrow C \cap \{x : \alpha^T x \ge \beta\}$.\\
5. & {\bf Termination}\\
   & If $\mathcal K = \emptyset$ and {\tt reinit} $= true$, then {\bf STOP}.\\
6. & {\bf Loop}\\
   & If $\mathcal K = \emptyset$ or tailing off is detected, then {\bf go to 1}, else {\tt reinit} $\leftarrow false$, {\bf go to 2.}
\end{tabular}
\caption{\label{alg:cut} Cutting plane algorithm for $\PLP$}
\end{algorithm}

A second aspect where we try to make our algorithm more efficient is in trying to make the solutions time of the problem \eqref{MP:sd} as small as possible. At each iteration of the algorithm in Step 4, we solve a sequence of such problem for a number of integer constrained variables $k_1,\ldots,k_p$. As can be expected, the problem are solved by the simplex algorithm, and a main question is which starting basis to use. Note that from one problem to the next, in general, all bounds on variables and constraints are changed as well as the objective function. Therefore, we can not expect the optimal basis of \MLP{k_{i}} to be either primal or dual feasible for \MLP{k_{i+1}}. There is one simple exception to this, if $\hat x_{k_{i}} = \hat x_{k_{i+1}}$, both problems have the same constraint system and therefore the solution of \MLP{k_{i}} is primal feasible for \MLP{k_{i+1}}. We try to benefit from this by solving the problems \eqref{MP:sd} in increasing order of the value $\hat x_k$ (also hoping that if $\hat x_{k_{i}}$ is close to $\hat x_{k_{i+1}}$, the solution of \MLP{k_{i}} will be close to being primal feasible for \MLP{k_{i+1}}). Again, we do not claim that our strategies are the best possible but just try to follow some sound heuristics rules, the computation in the next section will illustrate the practical performance.

\section{Computational results}
We now present some computational experiments aimed at assessing the interest of our procedures for optimizing over $\PSLP$ and approximating $\PLP$. The computations are performed with an implementation of Algorithm \ref{alg:cut} in C++. The code allows the use of two different LP solvers Clp \cite{Forrest:04} from COIN-OR \cite{COINOR} and IBM CPLEX. For efficiency reasons, the code calls each solver directly by using their respective native C++ and C API. All experiments are conducted on a machine equipped with Intel Quad Core Xeon 2.93GHz processors and 120 GiB of RAM, using only one thread for each run. We use Clp version 1.11 and Cplex version 12. The tolerance $\epsilon$ in Algorithm \ref{alg:cut} is set to $10^{-4}$ and a time limit of 1 hour is imposed. Note also that since the number of cuts generated can sometime grow very large, we use a cut pool to keep the size of the master problem reasonable.

Our main goal in these experiment is to assess how useful a relaxation $\PLP$ can be and how quickly it can be computed. First, we will present computations of $\PSLP$ and $\PLP$ on a large set of problems. Next, we present several comparisons with different methods:
\begin{itemize}
\item first, a comparison in terms of gap closed with textbook GMI cuts and lift-and-project cuts generated by rounds;
\item second, a comparison with several other methods recently proposed to generate rank-1 cuts: the heuristics proposed by Dash and Goycoolea \cite{Dash.Goycoolea:10}, the relax-and-cut framework proposed by Fischetti and Salvagnin \cite{Fischetti.Salvagnin:10}, the new reduce-and-split proposed by Cornu\'ejols and Nannicini \cite{Cornuejols.Nannicini:10} and finally the computations of the GMI/MIR closure made by Balas, Saxena \cite{Balas.Saxena:08} and Dash, G\"unl\"uk, Lodi \cite{Dash.Gunluk.Lodi:10};
\item finally, some preliminary experiments on using $\PLP$ in a branch-and-cut framework.
\end{itemize}
All experiments are performed on MIPLIB3 \cite{MIPLIB3} and MIPLIB 2003 \cite{Martin.Achterberg.ea:03} problems.

Before proceeding to the results, let us remind that computations of $\PLP$ were already conducted in \cite{Bonami.Minoux:05} and stress the differences with the current paper and the interest of the new experiments presented here. The approach in \cite{Bonami.Minoux:05} used a cut generation LP formulated in a higher dimensional space. Although it is possible that each cut separated was deeper there, the size of the cut generation LPs and their solution time limited the applicability of the approach to small and medium size problems. For that reason the computational experiments in \cite{Bonami.Minoux:05} were limited to problems with up to 1000 variables. We did not carry out systematic comparison between the two approaches but let us note that on the test set used in \cite{Bonami.Minoux:05} our code is about 12 times faster than the results reported there (although our tolerances are likely to be stricter). This is certainly in part due to faster machines and faster LP codes but we believe also to the use of a better separation LP. One main interest of the experiments here is that thanks to these faster computing time we are able to report numbers on the complete MIPLIB 3 and MIPLIB 2003 problems. 

\subsection{Computation of $\PSLP$ and $\PLP$}
\label{sec:comp_plp}
We ran our code to compute $\PSLP$ and approximate $\PLP$ on 62 MIPLIB 3.0 instances\footnote{We do not report computations on {\tt markshare1}, {\tt markshare2} and {\tt pk1}, because our code ran into computational troubles on these two instances when using CPLEX as LP solver. Note however that no gap is closed either by $\PSLP$ or $\PLP$} and 20 MIPLIB 2003 instances that are not already in MIPLIB 3.0. We have ran each instance with four variants of the code: computation of $\PSLP$  and approximation of $\PLP$ both with Clp and CPLEX as LP solver. The code is run on the original instance without any preprocessing. The results of these experiments are summarized in Tables \ref{tab:plp_miplib3} and \ref{tab:plp_miplib2003}. For each test problem, we report the CPU time and the fraction of the integrity gap closed by the method.

\topcaption{CPU time and gap closed by $\PSLP$ and $\PLP$ on MIPLIB 3.0.}\label{tab:plp_miplib3}
\tablefirsthead{\toprule
& \multicolumn{4}{c}{$\PSLP$} &
& \multicolumn{4}{c}{$\PLP$}
\\
\cmidrule{2-5}
\cmidrule{7-10}
& \multicolumn{2}{c}{\tt Clp}
& \multicolumn{2}{c}{\tt CPLEX} &
& \multicolumn{2}{c}{\tt Clp}
& \multicolumn{2}{c}{\tt CPLEX}
\\
name & CPU(s)                        & \% gap                          & CPU(s)                      & \% gap                       &  & CPU(s)                     & \% gap                      & CPU(s) & \% gap \\
\midrule
}
\tablehead{
\multicolumn{10}{c}{\normalsize \tablename\ \thetable{} (continued): CPU time and gap closed by $\PSLP$ and $\PLP$ on MIPLIB 3.0.}\\
\toprule
& \multicolumn{4}{c}{$\PSLP$}
& \multicolumn{4}{c}{$\PLP$}
\\
\cmidrule{2-5}
\cmidrule{7-10}
& \multicolumn{2}{c}{\tt Clp}
& \multicolumn{2}{c}{\tt CPLEX}&
& \multicolumn{2}{c}{\tt Clp}
& \multicolumn{2}{c}{\tt CPLEX}
\\
name & CPU(s)                        & \% gap                          & CPU(s)                      & \% gap                       & & CPU(s)                     & \% gap                      & CPU(s) & \% gap \\
\midrule
}
\tabletail{\bottomrule}
\tablelasttail{\bottomrule}

\begin{center}
\footnotesize
\begin{xtabular}{lrrrrp{1pt}rrrr}
10teams   & 3600   & 0       & 3600    & 31.31  & & 3600   & 100     & 3600    & 100     \\
air03     & 1.24   & 100     & 0.36    & 100    & & 0.86   & 100     & 0.39    & 100     \\
air04     & 3600   & 41.34   & 3600    & 83.71  & & 3600   & 53.66   & 3600    & 91.29   \\
air05     & 3600   & 49.98   & 3600    & 65.43  & & 3600   & 57.90   & 3600    & 68.11   \\
arki001   & 6.01   & 20.34   & 15.59   & 20.34  & & 3.42   & 35.59   & 4.51    & 36.44   \\
bell3a    & 0.01   & 64.56   & 0.02    & 64.56  & & 0.01   & 64.56   & 0.01    & 64.56   \\
bell5     & 0.06   & 86.25   & 0.04    & 86.25  & & 0.04   & 86.55   & 0.03    & 86.55   \\
blend2    & 0.11   & 21.82   & 0.08    & 21.82  & & 0.08   & 22.01   & 0.06    & 22.54   \\
cap6000   & 3.45   & 50      & 0.89    & 50     & & 1.30   & 62.50   & 0.25    & 56.25   \\
dano3mip  & 3600   & 0.20    & 3600    & 0.51   & & 3600   & 0.22    & 3600    & 0.52    \\
danoint   & 3600   & 4.75    & 245.25  & 5.57   & & 3600   & 5.06    & 213.79  & 5.97    \\
dcmulti   & 2.47   & 98.15   & 1.67    & 98.15  & & 1.94   & 98.76   & 0.85    & 98.19   \\
dsbmip    & 3600   & no\_gap & 0.60    & no\_gap& & 3600   & no\_gap & 0.98    & no\_gap \\
egout     & 0.05   & 93.85   & 0.03    & 93.85  & & 0.04   & 93.85   & 0.02    & 93.85   \\
enigma    & 1.06   & no\_gap & 3.92    & no\_gap& & 5.57   & no\_gap & 0.42    & no\_gap \\
fast0507  & 3600   & 3.01    & 3600    & 8.58   & & 3600   & 2.91    & 3600    & 11.35   \\
fiber     & 9.35   & 20.63   & 2.60    & 20.63  & & 1.41   & 93.55   & 0.59    & 97.10   \\
fixnet6   & 76.64  & 86.36   & 17.79   & 86.36  & & 37.88  & 86.53   & 11.32   & 87.12   \\
flugpl    & 0.00   & 11.72   & 0.00    & 11.72  & & 0.00   & 11.72   & 0.00    & 11.72   \\
gen       & 3.74   & 70.49   & 2.03    & 70.49  & & 0.38   & 82.51   & 0.20    & 93.99   \\
gesa2     & 2.10   & 59.10   & 1.38    & 59.10  & & 1.14   & 66.06   & 0.60    & 67.38   \\
gesa2\_o  & 2.23   & 59.80   & 1.75    & 59.80  & & 1.56   & 65.20   & 0.81    & 72.66   \\
gesa3     & 16.71  & 79.95   & 7.06    & 79.95  & & 3.19   & 94.40   & 1.19    & 94.59   \\
gesa3\_o  & 20.61  & 82.88   & 6.83    & 82.88  & & 2.75   & 94.40   & 1.38    & 94.40   \\
gt2       & 0.03   & 92.38   & 0.02    & 92.38  & & 0.04   & 98.34   & 0.02    & 98.58   \\
harp2     & 29.88  & 21.28   & 4.65    & 21.28  & & 7.27   & 50.39   & 5.02    & 56.03   \\
khb05250  & 0.52   & 99.86   & 0.32    & 99.86  & & 0.29   & 99.95   & 0.28    & 99.95   \\
l152lav   & 333.68 & 34.46   & 36.98   & 34.46  & & 200.71 & 61.64   & 20.55   & 64.53   \\
lseu      & 0.03   & 16.58   & 0.02    & 16.58  & & 0.02   & 77.45   & 0.01    & 69.59   \\
mas74     & 0.18   & 5.47    & 0.06    & 5.47   & & 0.05   & 7.95    & 0.04    & 8.06    \\
mas76     & 0.26   & 3.68    & 0.06    & 3.68   & & 0.05   & 7.32    & 0.02    & 7.32    \\
misc03    & 2.48   & 40.21   & 0.63    & 40.21  & & 1.64   & 40.21   & 0.49    & 40.21   \\
misc06    & 0.60   & 89.66   & 0.29    & 86.21  & & 0.48   & 95.40   & 0.21    & 87.36   \\
misc07    & 18.17  & 11.44   & 6.74    & 11.44  & & 14.18  & 13      & 4.71    & 13.31   \\
mitre     & 914.05 & 100     & 131.19  & 100    & & 6.61   & 100     & 3.48    & 100     \\
mkc       & 3600   & 50.60   & 26.72   & 63.18  & & 3600   & 62.82   & 84.80   & 64.03   \\
mod008    & 0.08   & 9.02    & 0.03    & 9.02   & & 0.04   & 37.36   & 0.02    & 37.58   \\
mod010    & 183.37 & 52.51   & 17.10   & 52.51  & & 1.86   & 100     & 0.11    & 100     \\
mod011    & 3600   & 39.41   & 3600    & 53.46  & & 3600   & 39.48   & 3600    & 56.24   \\
modglob   & 0.51   & 57.09   & 0.30    & 57.09  & & 0.72   & 57.09   & 0.47    & 61.91   \\
noswot    & 3600   & 0       & 0.06    & 0      & & 3600   & 0       & 0.08    & 0       \\
nw04      & 3600   & 35.19   & 3600    & 37.71  & & 3600   & 97.62   & 194.03  & 100     \\
p0033     & 0.01   & 8.19    & 0.01    & 8.19   & & 0.01   & 57.76   & 0.01    & 76.40   \\
p0201     & 0.75   & 46.85   & 0.61    & 46.85  & & 1.30   & 69.75   & 0.87    & 71.51   \\
p0282     & 0.50   & 93.90   & 0.21    & 93.90  & & 0.47   & 98.41   & 0.27    & 98.69   \\
p0548     & 2.66   & 91.35   & 1.65    & 91.35  & & 3.12   & 94.27   & 1.06    & 95.33   \\
p2756     & 5.72   & 81.59   & 2.89    & 81.60  & & 4.21   & 98.90   & 1.18    & 96.31   \\
pp08a     & 0.25   & 79.29   & 0.13    & 79.29  & & 0.26   & 79.29   & 0.15    & 79.29   \\
pp08aCUTS & 0.69   & 68.81   & 0.43    & 68.81  & & 0.76   & 70.17   & 0.55    & 69.09   \\
qiu       & 3600   & 76.20   & 979.10  & 78.10  & & 3600   & 69.63   & 1712.86 & 78.10   \\
qnet1     & 15.59  & 94.28   & 15.30   & 94.28  & & 22.77  & 94.49   & 15.48   & 96.27   \\
qnet1\_o  & 2.07   & 87.59   & 2.89    & 87.59  & & 2.08   & 89.74   & 4.03    & 89.78   \\
rentacar  & 3600   & 63.22   & 3600    & 90.28  & & 3600   & 75.68   & 3600    & 98.51   \\
rgn       & 0.18   & 11.88   & 0.21    & 11.88  & & 0.12   & 73.65   & 0.04    & 69.88   \\
rout      & 161.12 & 28.01   & 18.61   & 28.01  & & 12.63  & 52.18   & 2.40    & 55.24   \\
set1ch    & 0.64   & 39.88   & 0.42    & 39.88  & & 0.77   & 39.88   & 0.31    & 39.88   \\
seymour   & 3600   & 3.72    & 2624.25 & 55.94  & & 3600   & 4.46    & 3144.46 & 57.69   \\
stein27   & 0.29   & 0       & 0.33    & 0      & & 0.26   & 0       & 0.48    & 0       \\
stein45   & 8.72   & 0       & 23.90   & 0      & & 10.67  & 0       & 17.54   & 0       \\
swath     & 48.53  & 2.77    & 13.00   & 2.77   & & 22.07  & 26.31   & 5.52    & 26.31   \\
vpm1      & 0.18   & 31.42   & 0.11    & 31.42  & & 0.12   & 36.07   & 0.09    & 33.17   \\
vpm2      & 0.47   & 54.29   & 0.22    & 54.29  & & 0.40   & 54.36   & 0.18    & 54.41   \\
\end{xtabular}
\end{center}

\topcaption{CPU time and gap closed by $\PSLP$ and $\PLP$ for MIPLIB 2003.}
\label{tab:plp_miplib2003}
\tablefirsthead{\toprule
& \multicolumn{4}{c}{$\PSLP$} &
& \multicolumn{4}{c}{$\PLP$}
\\
\cmidrule{2-5}
\cmidrule{7-10}
& \multicolumn{2}{c}{\tt Clp}
& \multicolumn{2}{c}{\tt CPLEX} &
& \multicolumn{2}{c}{\tt Clp}
& \multicolumn{2}{c}{\tt CPLEX}
\\
name & CPU(s)                        & \% gap                          & CPU(s)                      & \% gap                       &  & CPU(s)                     & \% gap                      & CPU(s) & \% gap \\
\midrule
}
\tablehead{
\toprule
& \multicolumn{4}{c}{$\PSLP$}
& \multicolumn{4}{c}{$\PLP$}
\\
\cmidrule{2-5}
\cmidrule{7-10}
& \multicolumn{2}{c}{\tt Clp}
& \multicolumn{2}{c}{\tt CPLEX}&
& \multicolumn{2}{c}{\tt Clp}
& \multicolumn{2}{c}{\tt CPLEX}
\\
name & CPU(s)                        & \% gap                          & CPU(s)                      & \% gap                       & & CPU(s)                     & \% gap                      & CPU(s) & \% gap \\
\midrule
}
\tabletail{\bottomrule}
\tablelasttail{\bottomrule}

\begin{center}
\footnotesize
\begin{xtabular}{lrrrrp{1pt}rrrr}
a1c1s1       & 1509.64 & 78.68 & 141.93  & 78.75 & & 706.01  & 78.72 & 154.91  & 78.85 \\
aflow30a     & 8.19    & 42.41 & 2.42    & 42.41 & & 4.72    & 42.61 & 2.05    & 43.21 \\
aflow40b     & 69.58   & 34.29 & 14.02   & 34.29 & & 25.37   & 35.18 & 12.11   & 35.84 \\
atlanta-ip   & 3600    & 0.02  & 3600    & 1.71  & & 3600    & 0.02  & 3600    & 1.09  \\
glass4       & 13.99   & 0     & 1.01    & 0     & & 10.96   & 0.09  & 1.82    & 0.08  \\
momentum1    & 3600    & 30.47 & 3600    & 41.09 & & 3600    & 30.11 & 3600    & 42.80 \\
momentum2    & 3600    & 19.15 & 3600    & 41.96 & & 3600    & 11.69 & 203.17  & 41.34 \\
msc98-ip     & 3600    & 0.88  & 3600    & 42.26 & & 3600    & 1.19  & 3600    & 44.55 \\
mzzv11       & 3600    & 6.54  & 3600    & 51.23 & & 3600    & 10.97 & 3600    & 98.37 \\
mzzv42z      & 3600    & 4.53  & 3060.04 & 87.74 & & 3600    & 14.66 & 371.54  & 100   \\
net12        & 3600    & 5.71  & 3600    & 22.78 & & 3600    & 4.70  & 3600    & 22.62 \\
nsrand-ipx   & 1794.49 & 36.88 & 292.98  & 36.88 & & 63.55   & 75.38 & 28.55   & 77.70 \\
opt1217      & 2.32    & 0.19  & 0.78    & 0.19 & & 1.53    & 23.54 & 17.67   & 30.10 \\
protfold     & 3600    & 6.84  & 3600    & 10.29& & 3600    & 8.34  & 3600    & 7.80  \\
rd-rplusc-21 & 3600    & 0     & 3600    & 0    & & 3600    & 0     & 3600    & 0     \\
roll3000     & 721.17  & 16.30 & 163.90  & 16.30& & 547.61  & 54.11 & 139.12  & 56.27 \\
sp97ar       & 3600    & 42.04 & 570.42  & 42.06& & 3600    & 57.47 & 828.83  & 59.94 \\
timtab1      & 0.66    & 26.99 & 0.40    & 26.99& & 0.61    & 42.31 & 0.36    & 42.45 \\
timtab2      & 1.98    & 20.98 & 1.16    & 20.98& & 3.31    & 42.69 & 1.07    & 40.18 \\
tr12-30      & 7.91    & 64.12 & 4.16    & 64.12& & 6.49    & 64.13 & 3.33    & 64.12 \\
\end{xtabular}
\end{center}

Several remarks can be made from these results. First the proportion of the integrity gap closed is significant on most instances. There are only 5 instances ({\tt stein27, stein45, glass4} and {\tt rd-rplusc-21}) on which no gap is closed. On average, on MIPLIB 3.0, $\PSLP$ closes $47.12\%$ of the integrity gap with Clp and $50.40\%$ with CPLEX, while on MIPLIB 2003 the gap closed is $21.85$ and $33.1\%$ respectively (note that these differences are mostly due to problems which go to the time limit). Approximating $\PLP$ closes significantly more gap: on MIPLIB 3.0 $60.15\%$ and $63.42\%$ with Clp and CPLEX respectively, on MIPLIB 2003 $29.9\%$ and $44.37\%$. The computations on $\PLP$ highlight the heuristic nature of the approximation since the gap closed can be significantly different depending on the LP solver used (while the results are usually similar for $\PSLP$). Finally, computing times vary a lot between instances but are usually close between the four methods. In total, for $\PLP$ with CPLEX (our best method) 14 problems go to the time limit (7 in MIPLIB 3.0 and 7 in MIPLIB 2003). 7 out of these 14 can be solved to optimality in a shorter computing time with CPLEX ({\tt 10teams, air04, air05, fast0507, rentacar, qiu} and {\tt mzzv11}), 6 would take longer than an hour ({\tt atlanta-ip, momentum1, msc98-ip, net12}, {\tt protfold} and {\tt rd-rplus-sc21}) and one is an open problem ({\tt dano3mip}). In spite of these instances, computing times are typically  small with a geometric mean of 5.02 and 1.79 seconds respectively for $\PLP$ with Clp and CPLEX on MIPLIB 3.0, and $233$ and $119$ seconds respectively on MIPLIB 2003. Note that computing times are often slightly higher for $\PSLP$. This is not surprising since our stopping criterion is that the point $x^*$ belongs to $\PSLP$ and that the cuts generated in $\PLP$ are always deeper.

In table \ref{tab:details}, we give geometrical means for a number of key statistics of our procedure: the number of master iterations of the procedure (i.e. the number of times step 2 of Algorithm \ref{alg:cut} is performed), the CPU time spent to solve the master problem (i.e. the total time spent in step 2), the total number of time a separation problem \eqref{MP:sd} was solved, the number of times it led to a cut, the number of times it did not lead to a cut, the CPU time spent to solve separation problems, and finally the total number of simplex iterations for solving separation problems. Statistics refer to the runs approximating $\PLP$ with CPLEX as LP solver. Averages figures are given on three sets of problems: all instances of MIPLIB 3.0, all instances in MIPLIB 3.0 that took more than 1 second of CPU time and all instances solved within the time limit.
\begin{table}[htb]
\begin{center}
\begin{tabular}{lrrr}
\hline
& MIPLIB 3 & $>$ 1s. CPU& finished\\
\hline
\# Master iterations & 56.6 & 141.3 & 47.3\\
total resolve time & 0.4 & 17.43 & 0.14\\
\# \eqref{MP:sd} solved & 563.4 & 2844.5 & 433.10\\
\# cuts & 514 & 2749& 335.7\\
\# non-cut & 49.4 & 95.5 & 97.4\\
total \eqref{MP:sd} time & 0.5 & 14.21 & 0.21\\
total \eqref{MP:sd} pivots & 7,624 & 140,599 & 3,781\\
\hline
\end{tabular}
\end{center}
\caption{\label{tab:details} Detailed statistics of $\PLP$ approximation on MIPLIB3 problems.}
\end{table}

The main observation we can make from table \ref{tab:details} is that the time to solve \eqref{MP:sd} is usually very small: on average it takes $8.9 10^{-4}$ seconds to solve if we consider the whole of MIPLIB 3.0, $5.0 10^{-3}$ seconds if we restrict to problem that took more than 1 second of CPU time. This can be explained by the number of pivots which is also small with about 13 pivots on average on MIPLIB 3.0 (49 if we restrict to problems that took more that 1 sec.). This is also illustrated by the fact that in total re-solving the LP relaxations after adding cuts takes about the same time as solving \eqref{MP:sd}'s. Another interesting figure is the proportion of separation problems that led to a cut. Here one should only look at the statistics for problems that are finished (other statistics are biased by unfinished problems where the proportion of cuts is typically much higher), on average 22 \% of \eqref{MP:sd}'s did not lead to a cut. 

\subsection{Comparison with GMI and lift-and-project cuts}
We now turn to the comparison with GMI and lift-and-project cuts. To perform this comparison, we use a standard implementation of a GMI cuts (following \cite{Balas.Ceria.et.al:96}) generator that we developed using COIN-OR and the publicly available code for generating lift-and-project cuts {\tt CglLandP} \cite{CglLandP,Balas.Bonami:09}. We run both codes with different number of rounds. First we compare the gaps closed after one round of cuts. The goal of this is to compare the gap closed by rank-1 cuts using $\PSLP$ and $\PLP$ with the gap closed by the rank-1 cuts generated by more classical methods. Second, we compare the gaps closed after 10 rounds of cuts. This is a reasonable setting for cut generators in a branch and-cut procedure. Finally, we compare the gaps closed after 100 and 50 rounds of cuts for GMI and lift-and-project respectively. Since the computing time for $\PLP$ are rather high compared to 10 rounds, this last setting is to see if when given a large number of rounds GMI and lift-and-project cuts can approach $\PLP$. Of course, as more round are performed, the rank of the GMI and lift-and-project cuts increases (Fischetti, Lodi and Tramontani have experimented that it typically increases faster with GMI cuts \cite{Fischetti.Lodi.Tramontani:10}). Since both codes for GMI and lift-and-project cuts work only with {\tt Clp}, we compare with our runs using the same solver. In Table \ref{tab:textbook}, we report geometrical means of the computing times and averages of gap closed for each method on MIPLIB 3.0.

\begin{table}[hbt]
\begin{center}
\begin{tabular}{lrr}
\hline
& time (sec.) & \% gap closed\\
\hline
$\PSLP$ & 8.08 & 47.12\\
$\PLP$ & 5.02 & 60.15\\
1 GMI round & 0.00 & 25.74\\
10 GMI rounds & 0.04 & 44.39 \\
100 GMI rounds & 0.45 & 49.80\\
1 l-a-p round & 0.04 & 26.38 \\
10 l-a-p rounds & 1.63 & 52.88\\
50 l-a-p rounds & 9.36 & 57.62\\
\hline
\end{tabular}
\end{center}
\caption{\label{tab:textbook} Comparison between $\PSLP$ and $\PLP$ with GMI cuts and lift-and-project cuts generated by rounds. Geometrical means of CPU times and average gap closed.}
\end{table}

Several remarks are in order. First, comparing rank-1 cuts only, we note that the gap closed with $\PSLP$ and $\PLP$ is considerably bigger than the one closed by either GMI or lift-and-project cuts. Of course, this requires a much higher computing time. Iterating GMI and lift-and-project cuts, the gap closed never reaches the one closed by $\PLP$. After 100 rounds of GMI it is only slightly higher than the gap closed by $\PSLP$. The gap closed by lift-and-project cuts after 10 and 50 rounds is significantly bigger than with GMI cuts and with $\PSLP$ but it is still smaller than with $\PLP$, while the CPU time becomes bigger for 50 rounds. Due to the tailing off effect encountered when generating GMI and lift-and-project cut by rounds, it is likely that GMI generated in this fashion will never reach the gap closed by $\PLP$ and lift-and-project cut would probably have much difficulties.

\subsection{Comparison with other methods based on rank-1 cuts}
\label{sec:rank_ones}
To finish our comparisons of gaps closed, we compare our results with several recent related works that also aim at separating rank 1 cuts. First, we make a comparison with the computations of the split (or equivalently MIR) closure performed by Balas and Saxena \cite{Balas.Saxena:08} and Dash, G\"unl\"uk and Lodi \cite{Dash.Gunluk.Lodi:10}. The split closure is a tighter relaxation than ours, but the computing times are usually very high (hours). It is therefore interesting to know how close we are able to approach those results with our comparatively fast procedure. Secondly, we make a comparison with two methods that are also aimed at generating rank-1 GMI cuts from basic tableau's. The first of these was proposed by Dash and Goycoolea \cite{Dash.Goycoolea:10} and consists of several heuristics for finding violated rank-1 GMI cuts. The second of these is a framework based on generating GMI cuts and relaxing them in a Lagrangian fashion proposed by Fischetti and Salvagnin \cite{Fischetti.Salvagnin:10}. Finally, we make a comparison with another recent work by Cornu\'ejols and Nannicini \cite{Cornuejols.Nannicini:10} that aims at generating more rank-1 split cuts by a reduce-and-split approach. The test instances we consider are a subset of MIPLIB 3.0 and the MIPLIB 2003 that were used in \cite{Dash.Goycoolea:10,Fischetti.Salvagnin:10} (split closure computations are only available for MIPLIB 3.0).

In Table \ref{tab:rank-1-GMI}, we report the average gap closed by the split closure (for each problem we took the strongest value from \cite{Balas.Saxena:08} and \cite{Dash.Gunluk.Lodi:10}), Dash and Goycoolea method, Fischetti and Salvagnin method, and our method for $\PLP$. Both \cite{Dash.Goycoolea:10} and \cite{Fischetti.Salvagnin:10} have proposed several variants of their methods, to make the comparison concise we selected only one for each approach: the {\tt default} approach of Dash and Goycoolea and the {\tt fast} approach of Fischetti and Salvagnin. We also report indication of computing times. Note that these times should be considered with much care because experiments have been performed on different machines with different CPUs and different LP solvers. The computing times in \cite{Dash.Goycoolea:10} refer to a 1.4 GHz PowerPC machine for MIPLIB3.0 and to a 4.2 PowerPC machine for MIPLIB 2003. The computing times in \cite{Fischetti.Salvagnin:10}, refer to a 2.4GHz Intel Q6600.
As can be seen from the table, the split closure is significantly stronger than any of the three methods using GMI cuts from tableau's closing 79.2\% of the gap  on MIPLIB 3.0 versus between 59.6\% and 64.9\% for the three GMI methods. This can be seen as relatively far, but given the significantly shorter computing times of all three methods, it can also be seen as a fairly good approximation (note that on 9 instances our approach closes more gap than \cite{Balas.Saxena:08} and \cite{Dash.Gunluk.Lodi:10}). Comparing the three GMI methods, our approximation of $\PLP$ closes the most gap on MIPLIB 3.0 and \cite{Fischetti.Salvagnin:10} closes the most gap on MIPLIB 2003. The computing times of \cite{Fischetti.Salvagnin:10} and ours are close (\cite{Fischetti.Salvagnin:10} is slightly slower on MIPLIB 3, but quite faster on MIPLIB 2003), while \cite{Dash.Goycoolea:10} seems to be somewhat slower (note that besides the different CPUs, the three methods have different termination criterion). An important element not conveyed by average figures is that none of the three GMI methods dominates the other two: each closes the most gap on some instances. On MIPLIB 3.0, $\PLP$ closes more gap than the other two for 16 instances, \cite{Fischetti.Salvagnin:10} closes more gap for 10 and \cite{Dash.Goycoolea:10} closes more gap for 14. On MIPLIB 2003, $\PLP$ closes more gap for 7 instances, \cite{Fischetti.Salvagnin:10} for 8 and \cite{Dash.Goycoolea:10} for 4.
\begin{table}[hbt]
\begin{center}
\begin{tabular}{lrrrr}
\hline
method &\multicolumn{2}{c}{MIPLIB 3.0} &\multicolumn{2}{c}{MIPLIB 2003} \\
&  {\em time} \footnotemark[1] & \% gap & {\em time} \footnotemark[1]& \% gap\\
\hline
Split closure (best from \cite{Balas.Saxena:08,Dash.Gunluk.Lodi:10}) & 7556 & 79.2 & --- & ---\\
Heuristic rank-1 GMI cuts \cite{Dash.Goycoolea:10} &  20.05& 61.6 & 854 & 39.7\\
Relax and GMI cuts \cite{Fischetti.Salvagnin:10} ({\tt fast}) & 2.25 & 59.6 & 58.4 & 45.4 \\
$\PLP$ approximation & 1.36 & 64.4 & 119 & 44.4\\
\hline
\end{tabular}
\end{center}
\caption{\label{tab:rank-1-GMI}Comparison of methods for rank-1 GMI cuts from tableau's}.
\end{table}
\footnotetext[1]{timings on different machines}
Another method to compute split cuts that differs from traditional GMI cuts from the tableau is the reduce-and-split method \cite{Andersen.Cornuejols.ea:05}. The method was recently re-explored and enhanced by Cornu\'ejols and Nannicini \cite{Cornuejols.Nannicini:10}. In Table \ref{tab:red_split}, we compare the gap closed by our method to the gap closed in \cite{Cornuejols.Nannicini:10} and the Split Closure. The test set consists only of the Mixed models in MIPLIB 3.0 since reduce-and-split is particularly aimed at those. The gap closed reported in \cite{Cornuejols.Nannicini:10} correspond to the gap closed by one round of several cut generators for in the Cgl library (GMI cuts, MIR cuts, lift-and-project, two-step MIR, cover, flow covers) plus several variants of the reduce-and-split cuts. On these instances our method closes more gap on average than the combination of rank-1 cuts in \cite{Cornuejols.Nannicini:10}. Again it is important to note that our method does not dominate \cite{Cornuejols.Nannicini:10} and we close less gap on 9 instances.

\begin{table}[hbt]
\begin{center}
\begin{tabular}{lr}
\hline
method & \% gap \\
\hline
Split closure&78.64\\
$\PLP$&60.77 \\
Reduce and split + CglCuts \cite{Cornuejols.Nannicini:10}& 50.33\\
\hline
\end{tabular}
\end{center}
\caption{\label{tab:red_split} Comparison with reduce and split cuts on mixed models in MIPLIB 3.0.}
\end{table}

To conclude this section, we emphasis that all the rank-1 methods compared here complement each others. As noted several times, no method is dominating the others. Combining the four methods, even in a trivial way, should give a better approximation of the split closure.
\subsection{Branch-and-cut computations}
The experiments reported in the last three subsections show that our approach based on $\PLP$ seems competitive for building a strong relaxations at least for MIPLIB problems. Of course, our ultimate goal is to be able to solve those problems faster.
Applying our method in a branch-and-cut setting is not straightforward. First, the computing times, although small on average, are certainly higher than any of the cutting plane techniques already available in solvers. It can sometime be even larger than the total CPU time taken by a state-of-the-art solver such as CPLEX. Second, state-of-the art solvers combine many different cut generators and we need to include our method in the cut generation loop. Finally, although the gap closed is intuitively a good measure of the quality of a relaxation, many different factors can influence the performance in a branch-and-cut. 
Here, we present a preliminary attempt to deal with those issues and use $\PLP$ in a branch-and-cut setting. 

To setup the experiment, we use the callback system of CPLEX in order to include our procedure in its branch-and-cut loop. Note that our procedure can be included in two ways. Since the outer loop of our procedure is similar to a classical cut generation loop, we could try to include our procedure by just adding the inner loop (steps 3 and 4 of Algorithm \ref{alg:cut}) as a cut callback. This would have nevertheless the disadvantage that we would loose some control on the termination of the procedure. Therefore, we follow the more simple approach to run the complete procedure described by Algorithm \ref{alg:cut} (i.e. except for the context in which it is used it is the same procedure as before).

We call our procedure only once at the root node. Since, CPLEX has its own cutting plane procedures which are usually faster than ours, we should try as much as possible to benefit from them. In particular, even though we don't use the cuts generated by CPLEX in our separation LP, our procedure can benefit from CPLEX cutting planes because the point to cut should be closer to $\PLP$ after several rounds of cutting planes. For this reason, we try to call our procedure at the end of CPLEX cut generation loop. Since, as far as we know, there is no direct way of knowing if CPLEX has decided that it has exhausted its own cut generators, we attempt to detect it as follows: each time CPLEX calls our callback, we record the number of cuts generated for each class of cuts that CPLEX has, if from one iteration to the next none of those number changed, we launch our cutting plane procedure. At the end of the procedure, we pass to CPLEX all the cuts we found that are binding at the current optimum.

Note that many decisions that could impact the efficiency of our procedure in a branch-and-cut can be made. For example, cut generation procedures usually have policies to reject cuts that are too dense or deemed numerically unstable. Also, our termination criterion and the tolerances we use may be too tight for practical purposes. In this preliminary experiment, we do not intend to settle these issues. First, we want to keep the results in line with those of the previous sections. Second we believe that those issues can be addressed in different ways and should be the subject of further research. Our goal is just to give a preliminary view of how a procedure such as ours could be used, and what issues should be dealt with.

To test our procedure, we run it on the same set of MIPLIB 3.0 and MIPLIB 2003 problems that we used in Section \ref{sec:rank_ones}. We run two settings. First, the default CPLEX branch-and-cut algorithm (we don't use dynamic search since it is disabled when a cut callback is used in CPLEX) denoted {\tt CPLEX-BAC}. Second, CPLEX augmented with our cut callback denoted {\tt CPLEX}+$\PLP$. Note that, we apply our method to the preprocessed model of CPLEX. In both settings, we set a global time limit of 3 hours of CPU time (CPLEX is run on 1 thread). To report the results, we group the instances in three sets: {\bf set A} consists of 37 instances that are solved with both {\tt CPLEX-BAC} and {\tt CPLEX}+$\PLP$ in less than 10 seconds, {\bf set B} contains 16 instances that are not in set A but are solved by all our setups, finally {\bf set C} are 8 instances that could not be solved by any setup. Two instances do not belong to any of these sets and are reported separately: {\tt danonint} is solved by {\tt CPLEX-BAC} but not by {\tt CPLEX}+$\PLP$, {\tt nsrand-ipx} is solved by {\tt CPLEX}+$\PLP$ but not by {\tt CPLEX-BAC}. For {\tt CPLEX}+$\PLP$, we test two time limits for the $\PLP$ procedure a short time limit of 10 seconds (denoted {\tt CPLEX}+$\PLP$-10), and a long one of 1800 seconds (denoted {\tt CPLEX}+$\PLP$-1800).

We report the results in Table \ref{tab:bac}. For each set of instances, we report the average gap closed at the root, the geometrical means of CPU time and number of nodes processed and the average final gaps. 
\begin{table}
\begin{center}
\begin{tabular}{lrrrr}
\hline
method & root \% gap & time & \# nodes & final gap\\
\hline
\multicolumn{4}{c}{Instances in {\bf set A}}\\
\hline
{\tt CPLEX-BAC}& 73.2& 0.12 & 69.55 & 0\\
{\tt CPLEX}+$\PLP$& 85.9 & 0.39 & 41.04 & 0\\
\hline
\multicolumn{4}{c}{Instances in {\bf set B}}\\
\hline
{\tt CPLEX-BAC}& 32.4 & 46.64 & 11723 & 0\\
{{\tt CPLEX}+$\PLP$--10}& 46.0 & 93.85 & 10414 & 0\\
{{\tt CPLEX}+$\PLP$--1800}& 58.8 & 238.15 & 5461.31 & 0\\
\hline
\multicolumn{4}{c}{\tt danoint}\\
\hline
{\tt CPLEX-BAC}& 2.92 & 7993 & 1119303 & 0\\
{{\tt CPLEX}+$\PLP$--10}& 3.68 & $>$ 10800 & $>$ 585297 & 44.9\\
{{\tt CPLEX}+$\PLP$--1800}& 5.8 & $>$ 10800 & $>$ 188971 & 32.6\\
\hline
\multicolumn{4}{c}{\tt nsrand-ipx}\\
\hline
{\tt CPLEX-BAC}& 50.03& $>$ 10800& $>$849402 & 84\\
{{\tt CPLEX}+$\PLP$--10}& 76.8& 3449 & 220105 & 0\\
{{\tt CPLEX}+$\PLP$--1800}& 79.8 & 892.74 & 52780& 0\\
\hline
\multicolumn{4}{c}{Instances in {\bf set C}}\\
\hline
{\tt CPLEX-BAC}& 23.2 & $>$ 10800 & $>$ 31339 & 45.5\\
{{\tt CPLEX}+$\PLP$--10}& 24.7 & $>$ 10800  & $>$ 19066 & 42.7\\
{{\tt CPLEX}+$\PLP$--1800}& 28.7 & $>$ 10800  & $>$ 4383 & 42.5\\
\hline
\end{tabular}
\end{center}
\caption{\label{tab:bac}Summary of results with branch-and-cut.}
\end{table}

The result shows that we are still able to generate violated rank-1 cuts after CPLEX has generated all its cuts. Our methods allows to close $12.7\%$ more gap at the root for instances in set A, $13.6\%$ for instances in set B with a time limit of 10 seconds and $25.4\%$ with a time limit of 1800 seconds.  For instances in set C, the difference is less significant: $1.5\%$ and $5.5\%$ with a time limit of 10 and 1800 seconds respectively. Unfortunately, the improvement in terms of gap closed at the root does not carry out in terms of total solution times. For instances in set A the CPU time is more than 3 times bigger and for instances in set B it is around two time bigger with {{\tt CPLEX}+$\PLP$--10} and 5 times bigger for {{\tt CPLEX}+$\PLP$--1800}. For instances in set C, the gap closed after 3 hours is on average smaller with {{\tt CPLEX}+$\PLP$--10}. The results in terms of number of nodes are not as grim. The node reduction is limited with {{\tt CPLEX}+$\PLP$--10}, but with {{\tt CPLEX}+$\PLP$--1800} we are able to divided the total number of nodes by 2 on set B .

\section{Conclusions}
In this paper we proposed a procedure to approximate $\PLP$. Our approach seems competitive at least for strengthening the initial formulation and closing integrity gap and is also relatively fast compared to previous similar approaches. These results are another illustration of the strength of low rank cuts. The competitive CPU times are largely due to the separation procedure we use. As our experiment show, we are usually able to generate violated rank-1 cuts very quickly.

A possible source of improvement for the procedure in terms of CPU time, would be to use the separation oracle in a framework that would have better convergence properties than the Kelley cutting plane method. It is well known that the Kelley method has slow convergence. Recently, Fischetti and Salvagnin proposed an in-out scheme for optimizing $\PSLP$ which bears many similarities with our approach \cite{Fischetti.Salvagnin:10b}. The computational results they report do not seem competitive to ours in terms of CPU time but they are able to reduce significantly the number of iterations compared to a Kelley approach. These results g
ive hope that the CPU times to compute $\PSLP$ could still be significantly reduced (note however that adapting their approach does not seem trivial to us).

Another source of improvement would be to use the Balas Perregaard algorithm to strengthen the cut obtained by solving \eqref{MLP}. From a theoretical point of view, this is a straightforward thing to do. One just need to give the optimal basis of \eqref{MLP} as a starting point of the Balas Perregaard algorithm. Unfortunately, it is much more involved in practice, in particular if one want to keep the separation efficient in terms of CPU time.

Finally, we would like to emphasize the interest of finding an exact procedure to optimize $\PLP$. From the combination of our computational results and those contained in \cite{Dash.Goycoolea:10,Fischetti.Salvagnin:10}, we can try to get a feeling of how far our respective approximations of $\PLP$ are from being optimal. On the MIPLIB 3.0 test set, our respective average strengthening are $64.4\%$, $61.6\%$ and $59.6\%$ of the integrity gap closed. If we combine the three methods by taking the best result of the three methods on each instance the average goes up to $70.9\%$. This certainly shows that the three methods are still far from the value of $\PLP$. Note that even if we add up the computing times of the three procedures, the average CPU time is still a small fraction of the time to compute the split closure which closes $79.2\%$ of the integrity gap. As stated in introduction, as far as we know, the complexity of optimizing $\PLP$ is open. These numbers give a computational motivation for attempting to settle the question.


\end{document}